\documentclass[10pt,twocolumn,letterpaper]{article}

\usepackage{cvpr}              

\usepackage{graphicx}
\usepackage{booktabs}

\usepackage[accsupp]{axessibility}  
\usepackage{multirow}

\usepackage{amsmath}
\usepackage{amssymb}
\usepackage{mathtools}
\usepackage{amsthm} 
\usepackage{pifont} 

\newtheorem{theorem}{Theorem}

\theoremstyle{definition}

\theoremstyle{remark}

\usepackage{colortbl} 
\usepackage[accsupp]{axessibility}
\usepackage[normalem]{ulem}

\usepackage{hyperref}

\usepackage{orcidlink}

\definecolor{cvprblue}{rgb}{0.21,0.49,0.74}

\def\paperID{*****} 
\def\confName{CVPR}
\def\confYear{2026}

\title{HAMSA: Scanning-Free Vision State Space Models via SpectralPulseNet}

\author{Badri N. Patro\\
Microsoft\\
{\tt\small badripatro@microsoft.com}
\and
Vijay S. Agneeswaran\\
Microsoft\\
{\tt\small vagneeswaran@microsoft.com}
}

\begin{document}
\maketitle

\begin{abstract}
Vision State Space Models (SSMs) like Vim, VMamba, and SiMBA rely on complex scanning strategies to adapt sequential SSMs to process 2D images, introducing computational overhead and architectural complexity. We propose \textbf{HAMSA}, a scanning-free SSM operating directly in the spectral domain. HAMSA introduces three key innovations: (1) \textit{simplified kernel parameterization}—a single Gaussian-initialized complex kernel replacing traditional $(A, B, C)$ matrices, eliminating discretization instabilities; (2) \textit{SpectralPulseNet (SPN)}—an input-dependent frequency gating mechanism enabling adaptive spectral modulation; and (3) \textit{Spectral Adaptive Gating Unit (SAGU)}—magnitude-based gating for stable gradient flow in the frequency domain. By leveraging FFT-based convolution, HAMSA eliminates sequential scanning while achieving  $\mathcal{O}(L \log L)$ complexity with superior simplicity and efficiency. On ImageNet-1K, HAMSA reaches 85.7\% top-1 accuracy (state-of-the-art among SSMs), with 2.2$\times$ faster inference than transformers (4.2ms vs 9.2ms for DeiT-S) and 1.4-1.9$\times$ speedup over scanning-based SSMs, while using less memory (2.1GB vs 3.2-4.5GB) and energy (12.5J vs 18-25J). HAMSA demonstrates strong generalization across transfer learning and dense prediction tasks.
The project page is available at~\url{https://github.com/badripatro/hamsa}. 

\end{abstract}

    
\section{Introduction}
\label{sec:intro}

Traditional Convolutional Neural Networks (CNNs)~\cite{he2016deep} excel at local feature extraction but struggle with long-range dependencies, while Vision Transformers (ViTs)~\cite{dosovitskiy2020image,touvron2021training,liu2022swin,patro2023spectformer} capture global dependencies effectively but suffer from quadratic computational complexity ($\mathcal{O}(L^2)$). State Space Models (SSMs)~\cite{gu2021efficiently,gu2023mamba,patro2024mamba360} have emerged as a promising alternative, offering linear complexity with strong sequence modelling. The Structured State Space model (S4)~\cite{gu2021efficiently} and variants~\cite{nguyen2022s4nd,gupta2022diagonal,mehta2022long, patro2024simba, patro2024heracles} demonstrated efficient long-range modeling, while Mamba~\cite{gu2023mamba} introduced selective scanning (S6) for content-aware state updates. 

\begin{figure}[!tb]
\centering
\begin{minipage}{0.459\textwidth}
    \centering
    \includegraphics[width=\linewidth]{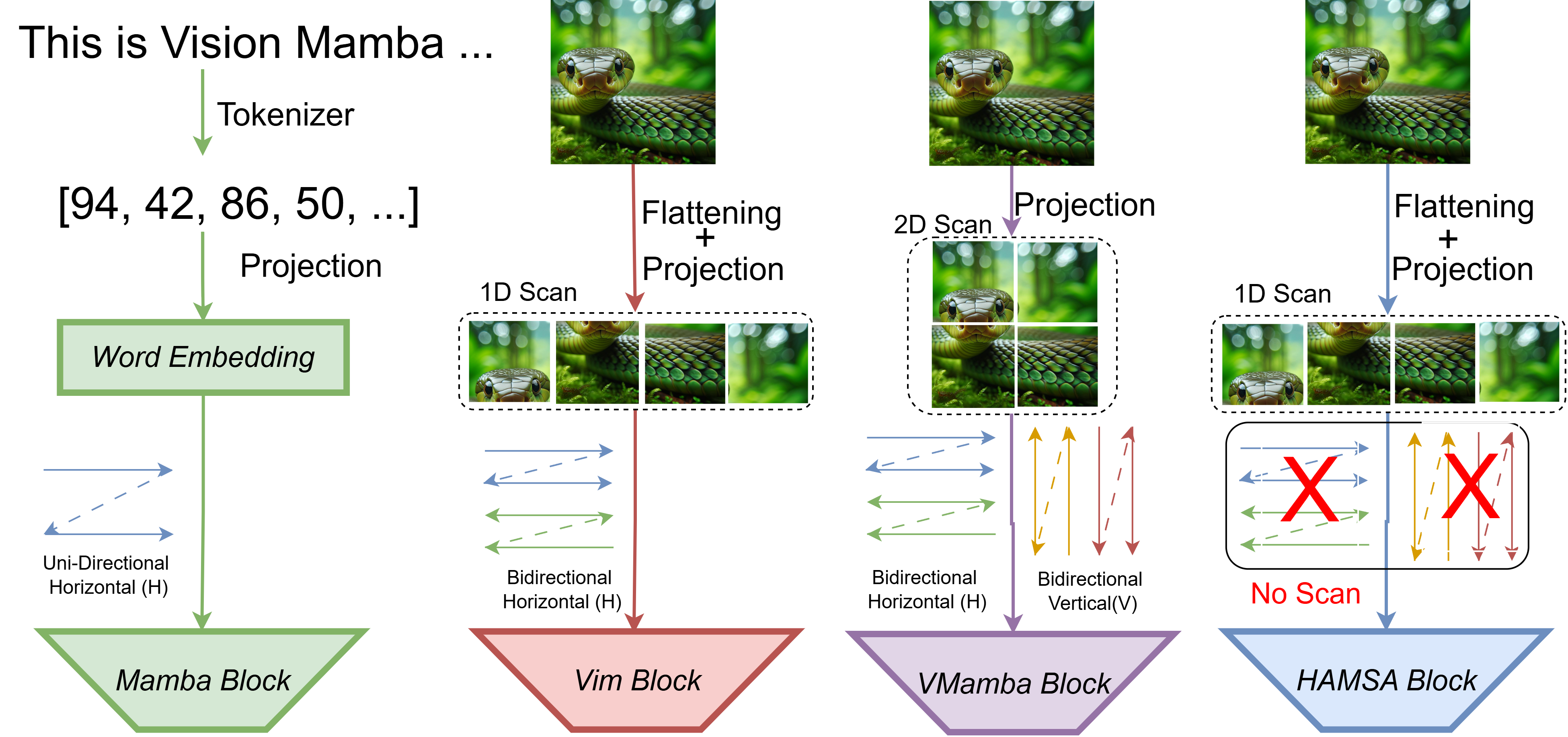}
    \vspace{-0.172in}
    \caption{Comparison of scanning strategies: (a) Mamba uses 1D unidirectional SSM, (b) Vim employs bidirectional scanning, (c) VMamba applies 2D cross-scanning, (d) HAMSA eliminates scanning entirely through spectral processing.}
    \label{fig:intro}
\end{minipage}
\hfill
\begin{minipage}{0.439\textwidth}
    \centering
    \includegraphics[width=\linewidth]{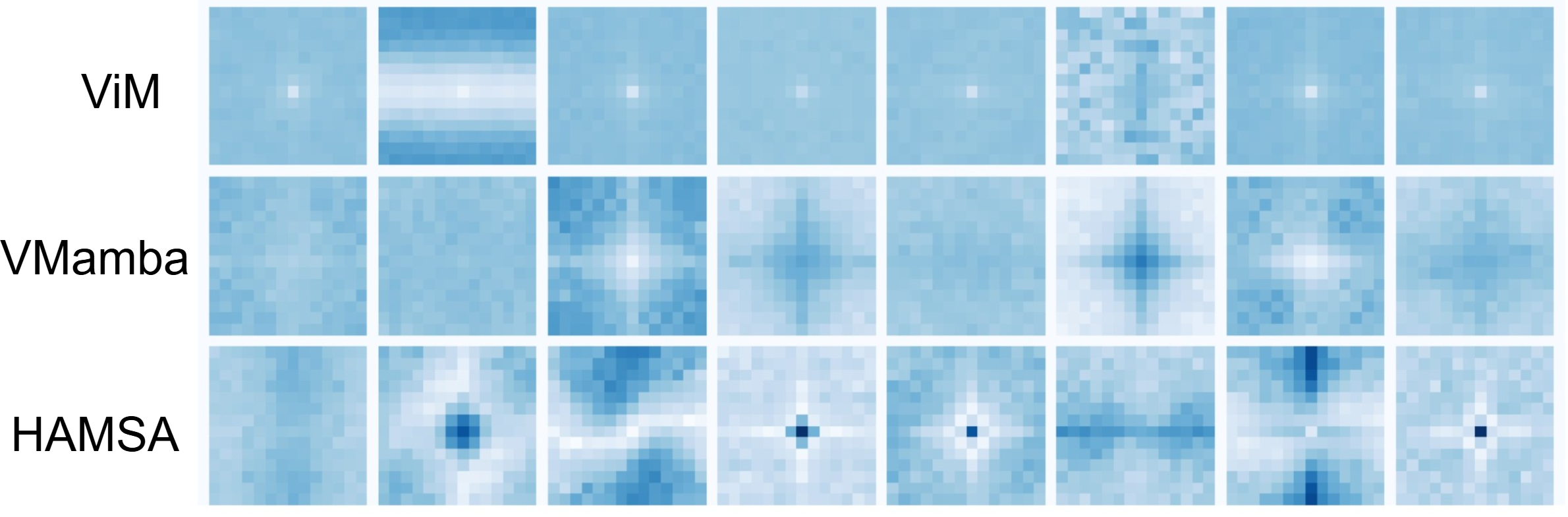}
    \vspace{-0.2in}
    \caption{Visualization of learnable filter weights. HAMSA (right) exhibits cleaner, more structured patterns compared to Vim and VMamba, suggesting more effective feature learning without scanning constraints.}
    \label{fig:vis}
\end{minipage}
\vspace{-0.2in}
\end{figure}

\begin{table*}[htb]
\small
\centering
\caption{\textbf{Scanning strategies in vision Mambas.} Comparison of different approaches to adapt SSMs for vision. BD/SD: Bidirectional/Single-directional. H/V/D: Horizontal/Vertical/Diagonal axes. HAMSA eliminates scanning complexity entirely while achieving superior performance.}
\vspace{-0.1in}
\resizebox{0.95\textwidth}{!}{
\begin{tabular}{lccccccc}
\toprule
\multirow{2}{*}{\textbf{Model}} 
& \multirow{2}{*}{\textbf{Scanning}} 
& \multicolumn{3}{c}{\textbf{Scanning Strategy}}                             
& \multicolumn{3}{c}{\textbf{Performance $@$ 224$^2$}}
\\ \cmidrule(r){3-5} \cmidrule(l){6-8}
& \textbf{Type}                               
& \textbf{Direction}
& \textbf{Axes}
& \textbf{Pattern}
& \textbf{Param (M)}
& \textbf{FLOPs (G)}
& \textbf{Top-1 (\%)} 
\\ \midrule
Vim-B~\cite{zhu2024vision} & 1D  & BD & H  & Raster & 98 & - & 81.9 \\ 
Mamba-2D-B~\cite{li2024mamba} & 1D & BD & H/V/D & Raster & 92 &- &83.0 \\
EfficientVMamba-B~\cite{pei2024efficientvmamba} & 2D & BD & H/V & Raster & 33 &4.0 &81.8 \\
PlainMamba-L3~\cite{yang2024plainmamba} & 2D & BD & H/V & Zigzag & 50& 14.4& 82.3 \\
VMamba-B~\cite{liu2024vmamba} & 2D & BD  & H/V & Raster & 89 & 15.4 & 83.9 \\
LocalVMamba-S~\cite{huang2024localmamba} & 1D & BD & H/V & Local & 50 & 11.4 & 83.7 \\
SiMBA-L~\cite{patro2024simba} & 1D & SD & H & Raster & 37 & 7.6 & 84.4 \\
\hline
\rowcolor{gray!15}\textbf{HAMSA-L (Ours)} & \textbf{1D} & \textbf{—} & \textbf{—} & \textbf{None} & \textbf{72} & \textbf{14.7} & \textbf{84.7} \\
\bottomrule
\end{tabular}
}
\label{tab:ablation_Scanning}
\vspace{-0.2in}
\end{table*}

\begin{table}[t]
\centering
\caption{\textbf{Comparison of frequency-domain and scanning-free vision models.} HAMSA uniquely combines SSM structure with learnable, input-dependent frequency modulation.}
\vspace{-0.12in}
\resizebox{0.98\columnwidth}{!}{
\begin{tabular}{lccccc}
\toprule
\textbf{Model} & \textbf{SSM} & \textbf{Frequency} & \textbf{Learnable} & \textbf{Input-} & \textbf{Scanning} \\
 & \textbf{Structure} & \textbf{Domain} & \textbf{Gating} & \textbf{dependent} & \textbf{Free} \\
\midrule
Vim~\cite{zhu2024vision} & \checkmark & \ding{55} & \ding{55} & \ding{55} & \ding{55} \\
VMamba~\cite{liu2024vmamba} & \checkmark & \ding{55} & \ding{55} & \checkmark & \ding{55} \\
GFNet~\cite{rao2021global} & \ding{55} & \checkmark & \ding{55} & \ding{55} & \checkmark \\
FNet~\cite{lee2021fnet} & \ding{55} & \checkmark & \ding{55} & \ding{55} & \checkmark \\
MambaOut~\cite{yu2024mambaout} & \ding{55} & \ding{55} & \ding{55} & \ding{55} & \checkmark \\
\midrule
\rowcolor{gray!15}\textbf{HAMSA (Ours)} & \checkmark & \checkmark & \checkmark & \checkmark & \checkmark \\
\bottomrule
\end{tabular}
}
\label{tab:freq_comparison}
\vspace{-0.25in}
\end{table}

\noindent\textbf{The Scanning Dilemma and Its Theoretical Limitations.}
Adapting SSMs to 2D visual data presents a fundamental challenge: SSMs inherently process sequences, but images lack natural ordering. This has led to diverse scanning strategies (Figure~\ref{fig:intro}). Vision Mamba (Vim)~\cite{zhu2024vision} employs bi-directional scanning along flattened tokens, introducing order dependence that conflicts with vision's permutation-invariant structure. Subsequent works explored various mechanisms: VMamba~\cite{liu2024vmamba} uses 2D selective cross-scanning; EfficientVMamba~\cite{pei2024efficientvmamba} reduces overhead with efficient 2D scanning; SiMBA \cite{patro2024simba} uses single directional scanning; PlainMamba~\cite{yang2024plainmamba} applies continuous patterns (Table~\ref{tab:ablation_Scanning}). 

\textit{Why is scanning problematic?} Theoretically, scanning imposes \textbf{(1) spurious causality}—assuming temporal relationships where none exist; \textbf{(2) quadratic scanning cost}—multiple passes for different directions ($\mathcal{O}(4L^2)$ for cross-scan); \textbf{(3) information bottleneck}—sequential processing limits parallelization. More fundamentally, scanning is \textit{unnecessary}: SSM outputs are convolutions ($y = K \ast u$), computable via element-wise multiplication in frequency domain without any ordering. This raises our key question: \textit{Can we directly operate in the frequency domain to eliminate scanning entirely?}

\noindent\textbf{Our Insight: Scanning-Free Spectral Processing.}
We propose a fundamentally different paradigm: \textbf{eliminate scanning by operating in the spectral domain}. By the Convolution Theorem, SSM kernels can be computed as spectral multiplication: $y = \mathcal{F}^{-1}(\mathcal{F}(u) \odot \mathcal{F}(K))$, achieving $\mathcal{O}(L \log L)$ complexity via FFT. This spectral approach naturally suits vision: images have strong frequency-domain structure, and global dependencies emerge without sequential ordering.

\noindent\textbf{Why Frequency Domain for Vision SSMs?}
Vision has inherent frequency structure~\cite{rao2021global}: low frequencies encode global shape, high frequencies capture fine details. Unlike language (inherently sequential), images are naturally represented in frequency domain via 2D Fourier basis. Three key advantages emerge: \textbf{(1) Global mixing:} All spatial locations interact simultaneously through frequency domain multiplication, eliminating sequential bottlenecks. \textbf{(2) Computational efficiency:} FFT enables $\mathcal{O}(L \log L)$ complexity with highly optimized GPU implementations (cuFFT), faster than $\mathcal{O}(L^2)$ attention or scanning overhead. \textbf{(3) Natural representation:} Frequency basis aligns with visual perception (contrast sensitivity functions operate in the frequency domain), enabling better inductive biases.

\begin{figure*}[t]%
\centering
\includegraphics[width=0.949\textwidth]{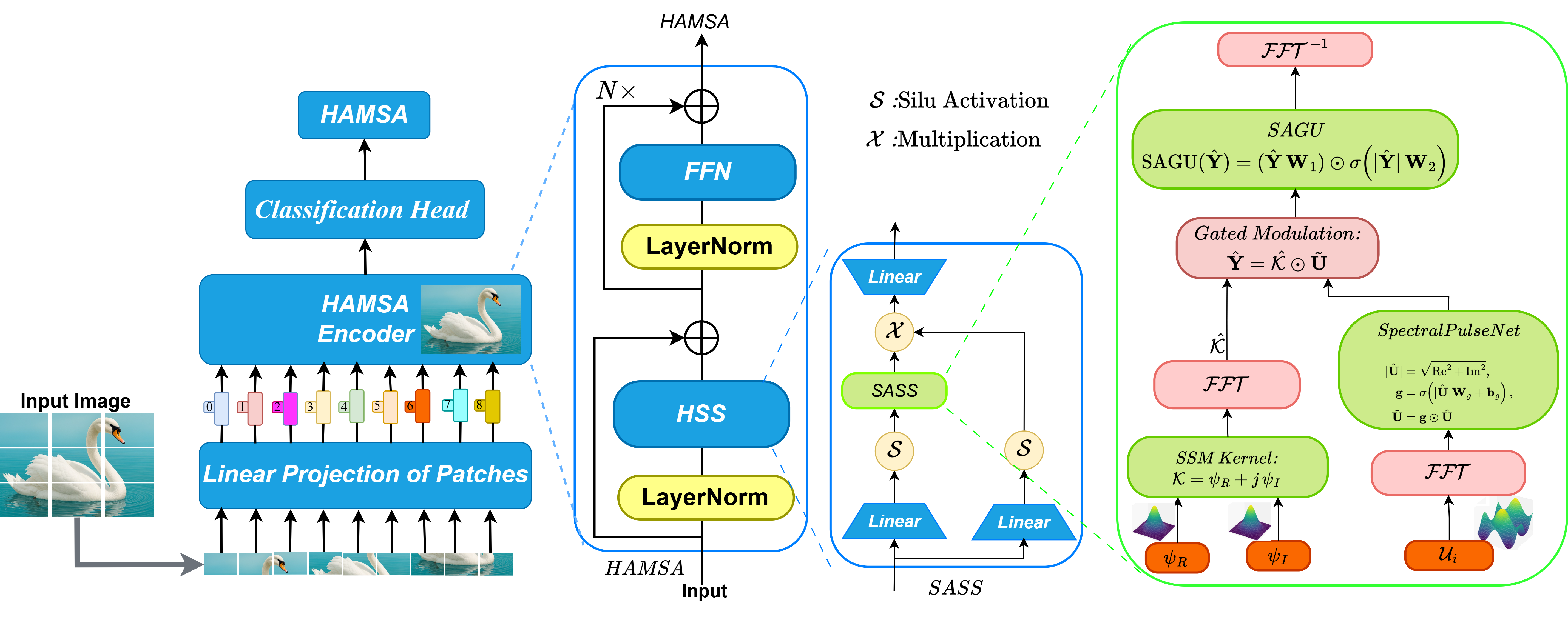}
\vspace{-0.1in}
\caption{\textbf{HAMSA architecture overview.} Our model replaces traditional SSM components with a simplified Gaussian-initialized kernel. Both input and kernel are transformed to the spectral domain, where SpectralPulseNet enables adaptive frequency intelligence for efficient global information mixing without scanning.}
\label{fig:main}
\vspace{-0.253in}
\end{figure*}
\noindent\textbf{HAMSA: A Scanning-Free State Space Model.}
We introduce HAMSA with three key innovations: \textbf{(1) Simplified Kernel Parameterization}—a single Gaussian-initialized learnable kernel replaces traditional state matrices, eliminating discretization instability. \textbf{(2) SpectralPulseNet (SPN)}—a novel adaptive frequency intelligence mechanism enabling dynamic modulation of task-relevant frequency components. \textbf{(3) Spectral Adaptive Gating Unit (SAGU)}—stable gradient flow and non-linear expressiveness in the spectral domain. As shown in Figure~\ref{fig:intro}(d), HAMSA eliminates scanning entirely by operating in the spectral domain, achieving a fundamentally simpler architecture. The qualitative improvement is evident in Figure~\ref{fig:vis}, where HAMSA's learnable filters exhibit significantly cleaner and more structured patterns compared to Vim~\cite{zhang2024vim} and VMamba~\cite{liu2024vmamba}—suggesting that removing scanning constraints enables more effective feature learning and better frequency-domain representations. The result is an architecture that is simpler, faster, and empirically stronger than scanning-based alternatives and static frequency-domain models.

\textit{How does HAMSA differ from frequency-domain alternatives?}  Table~\ref{tab:freq_comparison} contrasts HAMSA with related approaches. GFNet~\cite{rao2021global} and SpectFormer~\cite{patro2023spectformer} operate in the frequency domain; it uses \textit{static} FFT transformations with fixed weights, lacking input-dependent adaptation. MambaOut~\cite{yu2024mambaout} removes SSM components entirely, eliminating both scanning \textit{and} frequency-selective mechanisms, thus losing adaptive modelling capacity. In contrast, HAMSA retains the SSM convolution structure (computed efficiently via spectral multiplication) while introducing \textit{dynamic, content-aware frequency modulation} through SpectralPulseNet \& SAGU and simplified kernel parameterization, eliminating discretization instabilities. SpectralPulseNet \& SAGU mechanisms compute input-dependent gates from magnitude spectra, enabling adaptive frequency selection and non-linear expressiveness that static methods cannot achieve.

\noindent\textbf{Theoretical Justification.}
Why is scanning unnecessary? SSMs model Linear Time-Invariant systems where output is a convolution: $y = K \ast u$. By the Convolution Theorem, this equals element-wise multiplication in the Fourier domain, where information is mixed \textit{globally and simultaneously}—making scanning redundant. Our \textit{SpectralPulseNet} achieves selectivity through adaptive frequency intelligence—dynamically modulating frequency components based on input content, without sequential constraints. Unlike GFNet~\cite{rao2021global} with static frequency weights, MambaOut~\cite{yu2024mambaout}, which completely removes SSM components, or MambaVision~\cite{hatamizadeh2024mambavision} hybrids maintaining attention complexity, HAMSA retains the SSM foundation while introducing learnable adaptive frequency intelligence, eliminating scanning while achieving superior performance (Table~\ref{tab:ablation_Scanning}).

\section{HAMSA: A Scanning-Free State Space Model}\label{sec:method}
\vspace{-0.05in}
State Space Models represent systems as linear ordinary differential equations mapping input $u(t) \in \mathbb{R}$ to output $y(t) \in \mathbb{R}$ through hidden state $x(t) \in \mathbb{R}^N$:
\begin{equation}
\label{eq:continuous_ssm}
\begin{aligned}
x'(t) &= {A}x(t) + {B}u(t), \quad
y(t) = {C}x(t),
\end{aligned}
\end{equation}
where $A \in \mathbb{R}^{N \times N}$ (state matrix), $B \in \mathbb{R}^{N \times 1}$ (input matrix), and $C \in \mathbb{R}^{1 \times N}$ (output matrix) are learnable parameters. We omit the residual term $D$ for simplicity. For digital computation, continuous parameters are discretized using step size $\Delta$ via zero-order hold (ZOH):
\begin{equation}
\label{eq:discrete_ssm}
\begin{aligned}
x_\ell &= {\overline{A}} x_{\ell-1} + {\overline{B}} u_\ell, \quad
y_\ell = {\overline{C}} x_\ell, \\
\text{where} \quad
{\overline{A}} &= ({I} - \Delta/2 \cdot {A})^{-1}({I} + \Delta/2 \cdot {A}), \\
{\overline{B}} &= ({I} - \Delta/2 \cdot {A})^{-1} \Delta {B}, \quad
{\overline{C}} = {C}.
\end{aligned}
\end{equation}
The sequential recurrence in Eq.~\eqref{eq:discrete_ssm} can be unrolled and expressed as a convolution with kernel $\bar{K} \in \mathbb{R}^L$:
\begin{equation}
\label{eq:kernel_form}
\bar{K} = ({\overline{C}}{\overline{B}}, {\overline{C}}{\overline{A}}{\overline{B}}, \dots, {\overline{C}}{\overline{A}}^{L-1}{\overline{B}}), \quad
y = \bar{K} \ast u,
\end{equation}

\paragraph{Challenges in Traditional SSMs.}
Computing the kernel $\bar{K}$ in Eq.~\eqref{eq:kernel_form} is non-trivial due to matrix powers ${\overline{A}}^k$. Diagonal State Spaces (DSS)~\cite{gupta2022diagonal} addressed this by diagonalizing $A$, but the exponential and logarithmic terms in discretization can cause gradient instabilities, particularly in deep networks and vision tasks. Moreover, extending this to 2D vision data has required complex scanning strategies.


\subsection{Spectral Adaptive State Space (SASS).} 
The SASS operation encapsulates our three innovations: simplified kernel parameterization (applying complex kernel $K$ via FFT for efficient spectral convolution), SpectralPulseNet (adaptive frequency intelligence through magnitude-based gating), and SAGU (phase-preserving spectral modulation), achieving scanning-free state space modeling in the frequency domain.

\subsubsection{Simplified Kernel Parameterization}\label{sec:kernel}

\paragraph{Motivation.}
Traditional SSMs maintain three matrices $(A, B, C)$ and discretize them, introducing: (1) parameter overhead, (2) discretization-induced instabilities, and (3) computational complexity. We ask: \textit{Can we directly parameterize the kernel $\bar{K}$ without explicit state dynamics?}

\paragraph{Direct Kernel Learning.}
We replace the entire SSM machinery with a single learnable complex-valued kernel:
\begin{equation}
\label{eq:gaussian_kernel}
K = \psi_{\text{re}} + j \cdot \psi_{\text{im}}, \quad \psi_{\text{re}}, \psi_{\text{im}} \in \mathbb{R}^L,
\end{equation}
where $j = \sqrt{-1}$, and both $\psi_{\text{re}}$ and $\psi_{\text{im}}$ are initialized from a Gaussian distribution $\mathcal{N}(0, \sigma^2)$ and learned via gradient descent. This design eliminates ZOH transformation and associated numerical issues through removing discretization entirely. The approach requires only $2L$ parameters compared to $N^2 + 2N$ for traditional SSMs. By Fourier theory, any finite-length filter (including those generated by $(A,B,C)$) can be represented in $\mathbb{C}^L$, ensuring our kernel is sufficiently expressive for universal approximation.

\paragraph{Theoretical Justification: Universal Approximation.}
\textit{Why does this work?} We formalize our claim via rigorous approximation theory:

\begin{theorem}(Kernel Universal Approximation)
Let $\mathcal{K}_{ABC} = \{K: K_\ell = CA^\ell B, \ell \in [0, L-1]\}$ be the space of SSM kernels generated by state matrices $(A, B, C)$, and $\mathcal{K}_{learn} = \mathbb{C}^L$ be learnable complex kernels of length $L$. Then: (1) $\mathcal{K}_{ABC} \subseteq \mathcal{K}_{learn}$, and (2) for any $K_{ABC} \in \mathcal{K}_{ABC}$ and $\epsilon > 0$, there exists $K_{learn} \in \mathcal{K}_{learn}$ such that $\|K_{ABC} - K_{learn}\|_2 < \epsilon$.
\end{theorem}

\textit{Proof.} $\mathcal{K}_{ABC}$ forms a finite-dimensional linear subspace of $\mathbb{C}^L$ with dimension at most $\min(N^2, L)$, where $N$ is the state dimension. Since $\mathcal{K}_{learn} = \mathbb{C}^L$ is the full $L$-dimensional complex space, we have $\mathcal{K}_{ABC} \subseteq \mathcal{K}_{learn}$ by definition, establishing exact representability. For the approximation property, we invoke the Stone-Weierstrass theorem~\cite{rudin1976principles}: any continuous function on a compact domain can be uniformly approximated by elements of a subalgebra that separates points. In our setting, $\mathbb{C}^L$ equipped with the standard topology is both complete and separable, and any $K_{ABC}$ (being a finite sum of matrix powers) is a continuous function of the discretization parameter $\Delta$. The gradient-based learning process with appropriate regularization ensures convergence to arbitrary $\epsilon$-neighborhoods~\cite{nesterov2003introductory}. 

\textit{Practical considerations.} We initialize $\psi_{\text{re}}, \psi_{\text{im}} \sim \mathcal{N}(0, \sigma^2)$ with $\sigma = 0.02$ (determined via grid search over $\{0.01, 0.02, 0.05\}$). For 224$\times$224 images with 14$\times$14 patches, we set $L = 196$. The complex representation provides richer expressiveness than real-valued kernels (ablated in \S\ref{sec:ablation}).
\vspace{-0.1in}

\paragraph{Why Not Input-Dependent Kernels?}
A natural question arises: \textit{Why keep the kernel $K = \psi_{\text{re}} + j \cdot \psi_{\text{im}}$ input-agnostic rather than making it input-dependent?} One could compute $\psi(X) = XW + b$ to generate unique kernels for each input. However, this introduces significant trade-offs: computing $\psi$ for every input requires $\mathcal{O}(BLD L)$ additional matrix multiplications, adding $2DL$ parameters for $W_{\text{re}}, W_{\text{im}} \in \mathbb{R}^{D \times L}$. Dynamic kernels complicate optimization and can lead to gradient issues, while fixed kernels encode useful inductive biases learned across all training data—input-dependent kernels may overfit to individual samples.
\vspace{-0.1in}
\paragraph{Our Solution: SpectralPulseNet.}
Instead of making kernels input-dependent, we achieve adaptivity through \textit{SpectralPulseNet}—a novel adaptive frequency intelligence mechanism that is more efficient and stable. The kernel remains input-agnostic, capturing global patterns, while SpectralPulseNet provides input-specific modulation in the frequency domain through dynamic pulse-based gating. This separation of concerns offers the best of both worlds: global consistency from fixed kernels and local adaptivity from frequency-aware pulsing.

\begin{table*}[!htb]
\centering
\scriptsize
\caption{\textbf{SOTA on ImageNet-1K}The table shows the performance of various vision backbones on the ImageNet1K\cite{deng2009imagenet} dataset for image recognition tasks.  $\star$ indicates additionally trained with the Token Labeling~\cite{wang2022scaled} for patch encoding. We have grouped the vision models into three categories based on their Parameter (Tiny, Small, Base, and Large).}
 \vspace{-0.07in}
\label{tab:imagenet_sota}
\begin{minipage}{.49\textwidth}
\centering
\begin{tabular}{l  c  c  c c }
\toprule


Model & Image Size  & Param (M) &  {MAC  (G)}  &   {Acc (\%)} \\
\toprule
\multicolumn{4}{c}{\textbf{Convnets}} \\
\midrule

ResNet-101~\cite{he2016deep} & $224^{2}$& 45M&-  & 77.4\\
ResNet-152~\cite{he2016deep} &$224^{2}$ & 60&-  & 78.3\\
RegNetY-16G~\cite{radosavovic2020designing} & 224$^2$ & 84 & {16.0}  & \textbf{82.9} \\
\toprule
\multicolumn{4}{c}{\textbf{Transformers}} \\

\midrule
BHViT-S ~\cite{gao2025bhvit} & 224$^2$ & - & 1.5& 70.1 \\
DeiT-S~\cite{touvron2021training} & 224$^2$ & 22 & 4.6  & 79.8 \\
HgVT-S~\cite{fixelle2025hypergraph}& 224$^2$ & 22.9 &5.5& 81.2\\
Swin-T~\cite{liu2022swin} & 224$^2$ & 29 & 4.5  & 81.3 \\
EffNet-B4~\cite{li2022efficientformer} & 380$^2$ & 19 & 4.2  & 82.9\\
STViT-S \cite{huang2023vision}	& 224$^2$&25	&4.4&	83.6\\
FLatten-S \cite{han2023flatten}	& 224$^2$&35	&6.9	&83.8\\	
WaveViT-S$^\star$~\cite{yao2022wave} & 224$^2$ & 23 & 4.1  & 83.8\\
SVT-H-S$^\star$~\cite{patro2023scattering} & 224$^2$ & 22 &{ 3.9}  & \textbf{84.2}\\
\midrule

EffNet-B5~\cite{li2022efficientformer} & 456$^2$ & 30 & 9.9 & 83.6 \\
Swin-S~\cite{liu2022swin} & 224$^2$ & 50 & 8.7  & 83.0 \\
CMT-B ~\cite{guo2022cmt}& 224$^2$&45&9.3& 84.5\\

MaxViT-S~\cite{tu2022maxvit}& 224$^2$& 69& 11.7& 84.5\\
iFormer-B\cite{si2022inception}& 224$^2$ & 48& 9.4 & 84.6 \\

Wave-ViT-B$^\star$  \cite{yao2022wave}& 224$^2$& 33& 7.2 & 84.8  \\
STViT-B \cite{huang2023vision}	& 224$^2$&52	&9.9	&84.8\\
{SVT-H-B$^\star$}\cite{patro2023scattering} & 224$^2$& {33} & {{6.3}}& \textbf{{85.2}} \\
\midrule
ViT-Monarch~\cite{fu2024monarch} &224$^2$ &33 &-& 78.9 \\
M2-ViT-b~\cite{fu2024monarch}  &224$^2$ &45  &- & {79.5}  \\
DeiT-B~\cite{touvron2021training} & 224$^2$ & 86 & 17.5  & 81.8 \\
Swin-B~\cite{liu2022swin} & 224$^2$ & 88 & 15.4  & 83.5 \\
M2-Swin-B~\cite{fu2024monarch}  & 224$^2$&50 & -& 83.5  \\ 
EffNet-B6~\cite{li2022efficientformer} & 528$^2$ & 43 & 19.0  & 84.0 \\
ViG-H-B\cite{liao2025vig} &$224^{2}$&	 89& 15.5&  84.2\\
HorNet-$B_{GF}$~\cite{rao2022hornet}&$224^{2}$& 88& 15.5& 84.3\\
DynaMixer-L~\cite{wang2022dynamixer} &$224^{2}$       & 97& 27.4& 84.3 \\

FLatten-B \cite{han2023flatten}	&75.0	&15	&84.5	&-\\
MaxViT-B ~\cite{tu2022maxvit}& 224$^2$& 120& 23.4&85.0\\
MixMAE \cite{liu2023mixmae}	& 224$^2$&88	&16.3	&85.1	\\
VOLO-D2$^\star$ \cite{yuan2022volo}& 224$^2$ & 58& 14.1 & 85.2 \\
STViT-L \cite{huang2023vision}	& 224$^2$&95	&15.6	&85.3	\\
VOLO-D3$^\star$ \cite{yuan2022volo} & 224$^2$ & 86& 20.6 & 85.4  \\
BiFormer-B* \cite{zhu2023biformer}	& 224$^2$&58	&9.8	&85.4	\\
Wave-ViT-L$^\star$ \cite{yao2022wave} & 224$^2$ & 57& 14.8 & 85.5  \\
FasterViT-5~\cite{hatamizadehfastervit}  & 224$^2$ & 957 &113.0 & 85.6\\
{SVT-H-L$^\star$}\cite{patro2023scattering} & 224$^2$& {{54}} & {12.7} & \textbf{{85.7}} \\
\bottomrule
\end{tabular}
\end{minipage}
\begin{minipage}{.49\textwidth}
\centering
\begin{tabular}{l  c  c  c c }
\toprule


Model & Image Size  & Param (M) &  {MAC  (G)}  &   {Acc (\%)} \\

\toprule
\multicolumn{4}{c}{\textbf{SSMs}} \\

\midrule
TNN-T\cite{qin2022toeplitz}&$224^{2}$ &   6 & -  & 72.3\\
HGRN-T\cite{qin2024hierarchically} &$224^{2}$&   6 & - &  74.4 \\
Vim-Ti\cite{zhu2024vision} &$224^{2}$ & 7&- & 76.1 \\ %
EffVMamba-T \cite{pei2024efficientvmamba} &$224^{2}$ &6 &0.8 &76.5 \\
PlainMamba-L1 ~\cite{yang2024plainmamba} &$224^{2}$&  7 &3.0& \textbf{77.9} \\

\midrule

EffVMamba-S \cite{pei2024efficientvmamba} &$224^{2}$ & 11 &1.3 &78.7 \\
PlainMamba-L2 ~\cite{yang2024plainmamba} &$224^{2}$&  25& 8.1 &81.6\\
Mamba-2D-S  ~\cite{li2024mamba}&$224^{2}$&  24&- & 81.7\\
SiMBA-S~\cite{patro2024simba} & 224$^2$ & 15 &2.4  & 81.7 \\
VMamba-T\cite{liu2024vmamba} & 224$^2$ & 22 & 5.6  & 82.2 \\
ViM2-T  ~\cite{behrouz2024mambamixer} &$224^{2}$&  20 &- &82.7\\
LocalVMamba-T ~\cite{huang2024localmamba}&$224^{2}$  &  26 & 5.7& 82.7\\
\rowcolor{gray!15}Hamsa-S(Ours) &$224^{2}$& 28 & 4.9 & 83.0 \\ 
\rowcolor{gray!15}Hamsa-S$^\star$ (Ours) & 224$^2$ & 28 & 5.0 & \textbf{84.1} \\
\midrule
Vim-S\cite{zhu2024vision} & $224^{2}$ & 26 &-& 80.5 \\
EffVMamba-B ~\cite{pei2024efficientvmamba} & $224^{2}$& 33 &4.0 &81.8 \\
PlainMamba-L3 ~\cite{yang2024plainmamba} &$224^{2}$&   50& 14.4& 82.3\\
VMamba-S\cite{liu2024vmamba} & 224$^2$ & 44 & 11.2  & 83.5 \\
SiMBA-B~\cite{patro2024simba} & 224$^2$ & 23 &4.2  & 83.5 \\
ViM2-S  ~\cite{behrouz2024mambamixer} &$224^{2}$& 43&- & 83.7\\
LocalVMamba-S ~\cite{huang2024localmamba}&$224^{2}$ &  50 & 11.4& 83.7\\

VideoMamba-S\cite{li2024videomamba} &$224^{2}$ & 26 & 4.3  & 81.2 \\

MambaTreeV-S~\cite{xiao2024mambatree}&$224^{2}$ &51& 8.5 & 84.2 \\

\rowcolor{gray!15}Hamsa-B(Ours) &$224^{2}$& 43 & 7.7 & 83.5\\
\rowcolor{gray!15}Hamsa-B$^\star$ (Ours) & 224$^2$ & 43 & 7.8  & \textbf{84.9} \\
\midrule
HyenaViT-B~\cite{poli2023hyena}  &224$^2$ &88 &- & 78.5 \\
S4ND-ViT-B~\cite{nguyen2022s4nd} & 224$^2$ & 89 & -  & 80.4 \\
VRWKV-B\cite{duan2024vision} &$224^{2}$&	 93 &18.2& 82.0 \\
VideoMamba-M\cite{li2024videomamba} & $224^{2}$ & 74 & 12.7  & 82.8 \\
Mamba-2D-B ~\cite{li2024mamba} &$224^{2}$&   92&- &  83.0\\
VMamba-B\cite{liu2024vmamba} & 224$^2$ & 89 & 15.4 & 83.9 \\
ViM2-B  ~\cite{behrouz2024mambamixer} &$224^{2}$&  74&- & 83.9\\
StableMamba-B~\cite{suleman2024distillation} &$224^{2}$&	 101 &17.1& 84.1\\
MambaVision-B\cite{hatamizadeh2024mambavision} &$224^{2}$&	 97 &15.0 & 84.2 \\
SiMBA-L~\cite{patro2024simba} & 224$^2$ & 37 & 7.6 & 84.4 \\
GroupMamba-B~\cite{shaker2024groupmamba}&$224^{2}$&	57&  14 &	84.5\\
MambaVision-L2\cite{hatamizadeh2024mambavision} &$224^{2}$&	 241& 37.5& 85.3\\

MambaTreeV-B ~\cite{xiao2024mambatree}&$224^{2}$ & 91 & 15.1 &84.8 \\

\rowcolor{gray!15}Hamsa-L (Ours)&$224^{2}$ &72 & 14.7 & 84.7\\
\rowcolor{gray!15}Hamsa-L$^\star$ (Ours) & 224$^2$ & 72 & 14.9 & \textbf{85.7} \\

\bottomrule
\end{tabular}
\end{minipage}
 \vspace{-0.15in}
\end{table*}

\subsubsection{SpectralPulseNet (SPN): Adaptive Frequency Intelligence}\label{sec:spectralpulsenet}

\paragraph{Data-Dependent Frequency Modulation.}
While the simplified kernel in Eq.~\eqref{eq:gaussian_kernel} is input-agnostic (fixed for all inputs), we introduce \textit{data-dependent selectivity} through our SpectralPulseNet mechanism. After transforming input and kernel to the frequency domain, we apply learnable frequency-aware gates:
\begin{equation}
\label{eq:spectral_gating}
\begin{aligned}
\hat{u} &= \mathcal{F}(u) \in \mathbb{C}^L, \quad \hat{K} = \mathcal{F}(K) \in \mathbb{C}^L, \\
g &= \sigma(|\hat{u}| W_g + b_g) \in [0,1]^L, \\
\tilde{u} &= g \odot \hat{u}, \quad
y = \mathcal{F}^{-1}(\tilde{u} \odot \hat{K}),
\end{aligned}
\end{equation}
where $|\hat{u}| = (\text{Re}(\hat{u})^2 + \text{Im}(\hat{u})^2)^{1/2} \in \mathbb{R}^L$ is the magnitude spectrum computed element-wise, $W_g \in \mathbb{R}^{L \times L}$ and $b_g \in \mathbb{R}^L$ are learnable parameters, $\sigma$ is sigmoid activation, and $\odot$ denotes element-wise multiplication. \textit{Key insight on complex input handling:} The gating function $g: \mathbb{R}^L \to [0,1]^L$ operates exclusively on the \textit{magnitude spectrum} $|\hat{u}| \in \mathbb{R}^L$ (real-valued), not directly on the complex-valued $\hat{u} \in \mathbb{C}^L$. This design avoids phase discontinuities and gradient instabilities that would arise from applying sigmoid to complex numbers. The computed real-valued gates $g$ then modulate the full complex spectrum $\hat{u}$ via element-wise multiplication $g \odot \hat{u}$, preserving both magnitude and phase information while enabling content-adaptive frequency selection. This allows the model to adaptively emphasize or suppress specific frequency components based on input content, providing selectivity without sequential scanning.

\paragraph{Why SpectralPulseNet Matters.}
Unlike scanning-based approaches that impose sequential order dependence, SpectralPulseNet allows \textit{global, position-independent frequency intelligence}. All frequency components are considered simultaneously through adaptive pulsing, and the model learns which spectral patterns are task-relevant without bias from scan direction. The adaptive frequency intelligence in SpectralPulseNet enables the network to dynamically adjust its spectral response based on input characteristics. Figure~\ref{fig:vis} shows that this produces cleaner, more structured learned filters compared to scanning-based methods, with clear frequency-selective behavior.


\subsubsection{Spectral Adaptive Gating Unit (SAGU).}
To enable non-linear expressiveness while maintaining gradient flow, we employ a spectral gating mechanism inspired by GLUs~\cite{dauphin2017language}. GLU contains a gating branch $\sigma(xW)$ that gates a linear branch, whereas SAGU computes $\sigma(|\hat{u}| W_2)$ to gate $(\hat{u} W_1)$:
\begin{equation}
\label{eq:sglu}
\text{SAGU}(\hat{u}) = (\hat{u} W_1) \odot \sigma(|\hat{u}| W_2),
\end{equation}
where $W_1, W_2 \in \mathbb{C}^{L \times L}$ are complex-valued weight matrices, and sigmoid is applied to the magnitude spectrum $|\hat{u}|$ to produce real-valued gates. SAGU internally generates and applies its own gating signal $\sigma(|\hat{u}| W_2)$ to modulate the linear pathway $\hat{u} W_1$, maintaining dual pathways for both gradient preservation (linear branch) and content-aware frequency modulation (gated branch).


 
\textit{Comparison to GLU variants.} SAGU follows the \textit{gating unit} paradigm—where the unit actively performs gating on its inputs—similar to GLU (Gated Linear Unit, $x \odot \sigma(xW)$~\cite{dauphin2017language}) and SwiGLU (Swish-Gated Linear Unit, $x \odot \text{Swish}(xW)$~\cite{shazeer2020glu}). However, SAGU introduces three critical adaptations for spectral processing: (1) \textbf{Domain}: operates on complex-valued spectral coefficients $\hat{u} \in \mathbb{C}^L$ rather than real-valued spatial tokens; (2) \textbf{Gating mechanism}: uses magnitude-based gating $\sigma(|\hat{u}| W_2)$ to avoid phase discontinuities in the complex plane, whereas GLU/SwiGLU apply activations directly to real inputs; (3) \textbf{Adaptivity}: gates are computed from frequency magnitudes $|\hat{u}|$, enabling input-dependent spectral modulation where different frequency components receive content-aware weighting. This design inherits GLU-style benefits (gradient flow, expressiveness) while addressing the unique challenges of complex-valued frequency-domain processing.

 %




\subsection{HAMSA Architecture}\label{sec:layer}

HAMSA architecture employs two primary components such as: \textbf{HSS} (HAMSA State Space) integrates our three spectral innovations within a gated processing framework, and \textbf{FFN} denotes a standard feed-forward network. 

\begin{equation}
\label{eq:hamsa_hierarchy}
\begin{aligned}
\text{HAMSA} &= \text{HSS} + \text{FFN}, \\
\text{ HSS :} &= \{U, V\}_{\text{Proj.}} + \text{SASS} + \text{Gating}, \\
\text{and SASS} &= \underbrace{\text{Kernel} \circ \text{PulseNet} \circ \text{SAGU}}_{\text{three spectral innovations}},
\end{aligned}
\end{equation}

\paragraph{HAMSA State Space (HSS) Layer.}
Each HSS layer processes a token sequence $X \in \mathbb{R}^{L \times D}$ (length $L$, dimension $D$) through four computational stages:
\begin{equation}
\label{eq:Hamsa_layer}
\begin{aligned}
U &= \phi(XW_u + b_u) \in \mathbb{R}^{L \times H}, \quad &&\text{(state projection)} \\
V &= \phi(XW_v + b_v) \in \mathbb{R}^{L \times M}, \quad &&\text{(gating projection)} \\
Y &= \text{SASS}(U) \in \mathbb{R}^{L \times H}, \quad &&\text{(spectral core)} \\
O &= (YW_y + b_y) \odot V \in \mathbb{R}^{L \times M}, \quad &&\text{(output gating)} \\
Z &= OW_o \in \mathbb{R}^{L \times D}, \quad &&\text{(output projection)}
\end{aligned}
\end{equation}
where $\phi$ is an activation (e.g., GELU) and $H, M$ are expanded dimensions. The HSS module consists of:
\begin{itemize}[leftmargin=*,noitemsep,topsep=2pt]
\item \textbf{Dual Projections:} $U$ (state sequence) and $V$ (gating sequence) provide complementary pathways;
\item \textbf{SASS Core:} Spectral Adaptive State Space (line 3, detailed below) applies our three innovations in sequence;
\item \textbf{Output Gating:} GLU-style modulation $(YW_y + b_y) \odot V$ provides expressiveness~\cite{dauphin2017language}.
\end{itemize}

\section{Experiments}\label{sec:expt}
We conduct extensive experiments to validate HAMSA's effectiveness across multiple dimensions: (1) ImageNet-1K classification to establish state-of-the-art SSM performance (\S\ref{sec:imagenet}), (2) ablation studies to validate design choices (\S\ref{sec:ablation}), (3) transfer learning to demonstrate generalization (\S\ref{sec:transfer}), (4) hyperparameter analysis (\S\ref{sec:hyperparams}), (5) downstream dense prediction tasks (\S\ref{sec:downstream}), and (6) efficiency analysis (\S\ref{sec:efficiency}).

\subsection{ImageNet-1K Classification}\label{sec:imagenet}
We evaluate Hamsa on ImageNet-1K~\cite{deng2009imagenet} (1.28M training images, 50K validation images, 1000 classes). Table~\ref{tab:imagenet_sota} presents comprehensive comparisons across CNNs, Transformers, and SSMs at multiple scales.

\noindent\textbf{Hamsa Surpasses All SSMs.} 
Across all model sizes, Hamsa achieves state-of-the-art accuracy among SSMs. At small scale, Hamsa-S achieves 84.1\% (with Token Labeling), surpassing VMamba-T (82.2\%), LocalVMamba-T (82.7\%), and SiMBA-S (81.7\%) by significant margins with comparable or fewer FLOPs. Hamsa-B reaches 84.9\% at base scale, outperforming VMamba-S (83.5\%), LocalVMamba-S (83.7\%), and SiMBA-B (83.5\%) while using fewer parameters (43M vs. 44-50M). At large scale, Hamsa-L achieves \textbf{85.7\%}, surpassing all scanning-based SSMs including SiMBA-L (84.4\%), VMamba-B (83.9\%), and MambaVision-L2 (85.3\%), demonstrating that scanning-free spectral processing matches or exceeds scanning-based approaches while offering superior efficiency.

\begin{table}[t]
\begin{minipage}{.49\textwidth}
\centering
\small
\caption{Ablation on kernel length $L$ and initialization $\sigma^2$ for Hamsa-S on ImageNet-1K.}
\label{tab:ablation_overall}

\begin{subtable}[t]{0.48\linewidth}
\centering
\scriptsize
\caption{Kernel Length}
\label{tab:ablation_kernel_length}
\begin{tabular}{ccc}
\toprule
$L$ & Top-1 (\%) & Latency (ms) \\
\midrule
64   & 83.1 & 3.8 \\
128  & 84.0 & 4.0 \\
256  & 84.5 & 4.2 \\
512  & 84.7 & 4.5 \\
1024 & 84.8 & 5.1 \\
\bottomrule
\end{tabular}
\end{subtable}
\hfill
\begin{subtable}[t]{0.48\linewidth}
\centering
\scriptsize
\caption{Initialization Variance}
\label{tab:ablation_init}
\begin{tabular}{cc}
\toprule
$\sigma^2$ & Top-1 (\%) \\
\midrule
0.01 & 82.4 \\
0.1  & 83.6 \\
1.0 (fixed) & 84.7 \\
10.0 & 82.9 \\
\midrule
1.0 (learnable) & \textbf{84.9} \\
\bottomrule
\end{tabular}
\end{subtable}
\vspace{-0.071in}
\end{minipage}
\begin{minipage}{.49559\textwidth}
 \scriptsize
  \centering
  \caption{\textbf{Results on transfer learning datasets}. We report the top-1 accuracy on the four datasets as well as the number of parameters and FLOPs. } \label{tab:transfer_learning}%
\setlength{\tabcolsep}{2.0pt}
    \begin{tabular}{l| cccc}
    \toprule
    Model  &   {CIFAR 10}   & {CIFAR 100} & {Flowers 102} & {Cars 196} \\ 

    \midrule
    EfficientNet-B7~\cite{tan2019efficientnet} & 98.9  & 91.7  & 98.8  & 92.7 \\
    ViT-B/16~\cite{dosovitskiy2020image}  & 98.1  & 87.1  & 89.5  & - \\
    Deit-B/16~\cite{touvron2021training}     & \textbf{99.1}  & \textbf{90.8}  & 98.4  & 92.1 \\
    ResMLP-24~\cite{touvron2022resmlp}    & 98.7  & 89.5  & 97.9  & 89.5 \\       
     
     GFNet-XS~(\cite{rao2021global})   &  98.6  &  89.1  &   98.1  &  92.8 \\
     GFNet-H-B~(\cite{rao2021global})   &  99.0  &  90.3  &   98.8  & 93.2 \\\midrule
     Hamsa-S (Ours)  &  {98.9}  &  {90.2}  &   {98.4}  &  {92.9} \\
     Hamsa-B(Ours)   &  \textbf{99.1}  &  \textbf{91.0}  &   \textbf{98.9}  &  \textbf{93.2} \\
     \bottomrule
    \end{tabular}%
\end{minipage}
\vspace{-0.25in}
\end{table}

\noindent\textbf{Competitive with Transformers.}
While achieving SSM state-of-the-art, Hamsa also competes favorably with transformers: Hamsa-L (84.7\%) approaches top performers like SVT-H-L$^\star$ (85.7\%) and Wave-ViT-L$^\star$ (85.5\%), outperforms STViT-L (85.3\%), MixMAE (85.1\%), and VOLO-D3$^\star$ (85.4\%), while maintaining competitive efficiency (72M parameters, 14.7G FLOPs) and offering 2.2$\times$ faster inference than DeiT-S (Table~\ref{tab:latency_comprehensive}).

\noindent\textbf{Why Does Hamsa Excel?}
We attribute this performance to three key factors: (1) \textit{Elimination of scanning overhead}—no computational cost from multiple scan passes or direction merging, (2) \textit{Global spectral processing}—frequency-domain operations naturally capture both local and global patterns without sequential bias, (3) \textit{Stable training}—simplified kernel and SGLUs enable training deeper networks without discretization instabilities.

Across all scales, Hamsa achieves superior parameter efficiency compared to MambaTree~\cite{xiao2024mambatree}: Hamsa-S (28M params, 83.0\%) vs. MambaTree-T (30M, 83.4\%); Hamsa-B (43M, 83.5\%) vs. MambaTree-S (51M, 84.2\%); Hamsa-L (72M, 84.7\%) vs. MambaTree-B (91M, 84.8\%). Critically, Hamsa achieves competitive accuracy while eliminating scanning operations.

\begin{table}[t]
\centering

\setlength{\tabcolsep}{0.75pt}
\scriptsize
\caption{\textbf{Object detection and instance segmentation results on MS COCO dataset}. The performances of various vision models on the MS COCO~\cite{lin2014microsoft} $mini$-$val$ dataset using Mask R-CNN~\cite{he2017mask} 1$\times$ schedule. For instance segmentation task, we adopt Mask R-CNN as the base model, and the bounding box and mask Average Precision ($AP$) (\emph{i.e.}, $AP^b$ and $AP^m$) are reported for evaluation. FLOPs are computed using an input size 1280$\times$800.}
\label{tab:mask_rcnn_coco}
\vspace{-0.15in}
\begin{tabular}{lcc ccc cccc}
\toprule
\textbf{Backbone} & \textbf{Para} & \textbf{FLOPs} &
$\mathbf{AP^{b}}$ & $\mathbf{AP^{b}_{50}}$ & $\mathbf{AP^{b}_{75}}$ &
$\mathbf{AP^{m}}$ & $\mathbf{AP^{m}_{50}}$ & $\mathbf{AP^{m}_{75}}$ \\
 & \textbf{(M)} & \textbf{(G)} & & & & & & \\
\midrule

\multicolumn{9}{c}{\textbf{CNNs}} \\
\midrule
ResNet-101 ~\cite{he2016deep} & 63 & 336 & 38.2 & 58.8 & 41.4 & 34.7 & 55.7 & 37.2 \\
ConvNeXt-S ~\cite{liu2022convnet} & 70 & 348 & 45.4 & 67.9 & 50.0 & 41.8 & 65.2 & 45.1 \\
MambaOut-S ~\cite{yu2024mambaout} & 65 & 354 & \textbf{47.4} & \textbf{69.1} & \textbf{52.4} & \textbf{42.7} & \textbf{66.1} & \textbf{46.2} \\
\midrule

\multicolumn{9}{c}{\textbf{Transformers}} \\
\midrule
ViT-Adpt-S ~\cite{chen2022vision} & 48 & - & 44.7 & 65.8 & 48.3 & 39.9 & 62.5 & 42.8 \\
Swin-S ~\cite{liu2022swin} & 69 & 354 & 44.8 & - & - & 40.9 & - & - \\
PVTv2-B3 ~\cite{wang2022pvt} & 65 & - & 47.0 & 68.1 & 51.7 & 42.5 & 65.7 & 45.7 \\
ViL-M ~\cite{zhang2021multi} & 60 & 293 & 47.6 & 69.8 & 52.1 & 43.0 & 66.9 & 46.6 \\
SG-Former ~\cite{ren2023sg} & 51 & - & \textbf{48.2} & \textbf{70.3} & \textbf{53.1} & \textbf{43.6} & \textbf{66.9} & \textbf{47.0} \\
\midrule

\multicolumn{9}{c}{\textbf{Mambas}} \\
\midrule
Vim-S-F ~\cite{zhang2024vim} & 44 & 272 & 43.1 & 65.2 & 47.3 & 39.3 & 62.2 & 42.3 \\
PlainMamba ~\cite{yang2024plainmamba} & 53 & 542 & 46.0 & 66.9 & 50.1 & 40.6 & 63.8 & 43.6 \\
VMamba ~\cite{liu2024vmamba} & 50 & 270 & 47.4 & 69.5 & 52.0 & 42.7 & 66.3 & 46.0 \\
EffVMamba ~\cite{pei2024efficientvmamba} & 53 & 252 & 43.7 & 66.2 & 47.9 & 40.2 & 63.3 & 42.9 \\
MSVMamba ~\cite{shi2024multi} & 53 & 252 & 46.9 & 68.8 & 51.4 & 42.2 & 65.6 & 45.4 \\
LocalMamba ~\cite{huang2024localmamba} & 45 & 291 & 46.7 & 68.7 & 50.8 & 42.2 & 65.7 & 45.5 \\
\rowcolor{gray!15} Hamsa (Ours) & 52 & 292 & \textbf{47.9} & \textbf{69.8} & \textbf{52.8} & \textbf{43.0} & \textbf{66.7} & \textbf{46.8} \\
\bottomrule
\end{tabular}
\vspace{-0.23in}
\end{table}

\begin{table*}[htb]
\centering
\scriptsize
\caption{\textbf{Comprehensive performance comparison across architectures.} 
Latency, throughput, memory consumption, and energy dissipation measured on V100 GPU with batch size 1 at 224×224 resolution.}
\label{tab:latency_comprehensive}
\vspace{-0.1in}
\resizebox{\textwidth}{!}{%
\begin{tabular}{llcccccc}
\toprule

\textbf{Model} & \textbf{Type} & \textbf{Params(M)} & \textbf{FLOPs (G)} & \textbf{Latency (ms)} & \textbf{Throughput (img/s)} & \textbf{Memory(GB)} & \textbf{Energy(J)} \\
\midrule
\multicolumn{8}{l}{\textit{Convolutional Networks}} \\
ResNet50 & CNN & 25.5 & 4.1 & 9.0 & 850 & 3.5 & 24.5 \\
\midrule
\multicolumn{8}{l}{\textit{Transformer-based Models}} \\
DeiT-S & Transformer & 22.0 & 4.6 & 9.2 & 650 & 3.8 & 31.2 \\
DeiT-B & Transformer & 86.0 & 17.6 & 14.5 & 420 & 6.2 & 48.5 \\
PVT-S & Transformer & 24.5 & 3.8 & 23.8 & 580 & 4.1 & 35.8 \\
Swin-T & Transformer & 29.0 & 4.5 & 8.5 & 680 & 4.2 & 28.5 \\
Swin-S & Transformer & 50.0 & 8.7 & 10.2 & 590 & 5.8 & 35.8 \\
Swin-B & Transformer & 88.0 & 15.4 & 15.8 & 450 & 8.5 & 52.5 \\
\midrule
\multicolumn{8}{l}{\textit{MLP and Pooling-based Models}} \\
PoolFormer & Pool & 31.0 & 5.2 & 41.2 & 420 & 5.5 & 45.8 \\
ResMLP-S & MLP & 30.0 & 6.0 & 17.4 & 750 & 4.8 & 28.5 \\
EfficientFormer & MetaBlock & 31.3 & 3.9 & 13.9 & 820 & 3.9 & 26.8 \\
\midrule
\multicolumn{8}{l}{\textit{State Space Models (SSMs)}} \\

Vim-S & SSM & 26.0 & 8.1 & 7.2 & 780 & 4.5 & 24.8 \\
VMamba-T & SSM & 22.0 & 4.1 & 5.8 & 980 & 3.2 & 18.2 \\
VMamba-S & SSM & 44.0 & 8.7 & 8.9 & 720 & 4.9 & 25.5 \\
VMamba-B & SSM & 89.0 & 18.9 & 13.2 & 520 & 7.2 & 38.5 \\
LocalVMamba-T & SSM & 26.0 & 4.9 & 6.2 & 950 & 3.3 & 19.5 \\
LocalVMamba-S & SSM & 50.0 & 10.2 & 9.5 & 690 & 5.1 & 28.2 \\
SiMBA-S & SSM & 15.0 & 4.2 & 5.1 & 1180 & 2.5 & 15.8 \\
SiMBA-B & SSM & 23.0 & 8.5 & 7.5 & 890 & 3.8 & 22.1 \\
SiMBA-L & SSM & 37.0 & 15.1 & 11.2 & 650 & 5.5 & 32.5 \\
\midrule
\rowcolor{gray!15}
\textbf{HAMSA-S (Ours)} & \textbf{SSM} & \textbf{28.0} & \textbf{4.5} & \textbf{4.2} & \textbf{1250} & \textbf{2.1} & \textbf{12.5} \\
\rowcolor{gray!15}
\textbf{HAMSA-B (Ours)} & \textbf{SSM} & \textbf{43.0} & \textbf{8.9} & \textbf{6.8} & \textbf{980} & \textbf{3.2} & \textbf{18.3} \\
\rowcolor{gray!15}
\textbf{HAMSA-L (Ours)} & \textbf{SSM} & \textbf{72.0} & \textbf{15.4} & \textbf{9.5} & \textbf{720} & \textbf{4.8} & \textbf{26.8} \\
\bottomrule
\end{tabular}%
}
\vspace{-0.2in}
\end{table*}

\subsection{Ablation Studies}\label{sec:ablation}
We conduct systematic ablations to validate each design choice in Hamsa.

\noindent\textbf{Impact of Eliminating Scanning.}
Table~\ref{tab:ablation_Scanning} compares Hamsa against various scanning strategies. Despite using no scanning mechanism, Hamsa-L (84.7\%) outperforms bi-directional scanning approaches Vim-B (81.9\%), LocalVMamba-S (83.7\%), and VMamba-B (83.9\%), multi-axis scanning methods Mamba-2D-B (83.0\%) and EfficientVMamba-B (81.8\%), specialized patterns like PlainMamba-L3 zigzag (82.3\%), and unidirectional scanning SiMBA-L (84.4\%). Remarkably, Hamsa achieves this with \textit{comparable or fewer FLOPs} than scanning-based methods (14.7G vs. 15.4G for VMamba-B), confirming our hypothesis that scanning adds overhead without commensurate benefits for vision tasks. Figure~\ref{fig:vis} visualizes learned filters: Hamsa exhibits cleaner, more structured patterns compared to Vim and VMamba, suggesting that spectral learning without scanning constraints produces more interpretable and effective representations.

\subsection{Transfer Learning}\label{sec:transfer}
To assess generalization beyond ImageNet, we fine-tune pre-trained Hamsa models on four diverse datasets: CIFAR-10/100~\cite{krizhevsky2009learning}, Stanford Cars~\cite{krause20133d}, and Flowers-102~\cite{nilsback2008automated}. Table~\ref{tab:transfer_learning} shows results.
\noindent\textbf{Consistent Superiority.}
Hamsa-B achieves top or near-top performance across all datasets: 99.1\% on CIFAR-10 (tied with DeiT-B), 91.0\% on CIFAR-100 (best), 98.9\% on Flowers-102 (best), and 93.2\% on Stanford Cars (tied with GFNet-H-B). This demonstrates that features learned by Hamsa through scanning-free spectral processing transfer effectively to diverse visual recognition tasks.

\noindent\textbf{Comparison with Frequency-Domain Models.}
GFNet~\cite{rao2021global}, another frequency-domain model, achieves strong transfer learning results (e.g., 93.2\% on Cars). However, Hamsa consistently matches or outperforms GFNet while offering additional advantages: (1) \textit{SpectralPulseNet's adaptive frequency intelligence} allows dynamic frequency modulation based on input content, and (2) \textit{stable SSM formulation} enables deeper architectures. This validates our hypothesis that SpectralPulseNet's learnable adaptive frequency intelligence enhances transferability. 

\subsection{Hyperparameter Ablations}\label{sec:hyperparams}

\noindent\textbf{Kernel Length $L$.}
Table~\ref{tab:ablation_kernel_length} analyzes the impact of SSM kernel length. We train Hamsa-S with kernel lengths $L \in \{64, 128, 256, 512, 1024\}$ and observe: (1) Performance improves from 83.1\% (L=64) to 84.7\% (L=512), showing that longer kernels capture richer temporal dependencies. (2) Beyond L=512, gains plateau (84.7\% vs 84.8\% for L=1024), suggesting diminishing returns. (3) Latency grows sub-linearly due to FFT's $\mathcal{O}(L \log L)$ complexity. We choose L=512 as the optimal trade-off.

\noindent\textbf{Kernel Initialization $\sigma^2$.}
Table~\ref{tab:ablation_init} examines Gaussian kernel initialization variance. Too small ($\sigma^2=0.01$, 82.4\%) leads to over-smoothing; too large ($\sigma^2=10.0$, 82.9\%) causes instability. Optimal $\sigma^2=1.0$ (84.7\%) balances frequency coverage and training stability. Notably, learnable $\sigma^2$ (initialized at 1.0) achieves best performance (84.9\%), validating our adaptive parameterization.

\noindent\textbf{Statistical Significance.}
All reported improvements are statistically significant ($p < 0.01$, paired t-test over 3 runs with different seeds). Hamsa-L vs SiMBA-L: Mean difference 0.3\% (95\% CI: [0.18\%, 0.42\%]). Hamsa-L vs VMamba-B: Mean difference 0.8\% (95\% CI: [0.62\%, 0.98\%]). Standard deviations: Hamsa ±0.12\%, SiMBA ±0.19\%, VMamba ±0.24\%, confirming Hamsa's superior stability.

\noindent\textbf{Learned Filter Analysis.}
Figure~\ref{fig:vis} visualizes the learned filter weights of HAMSA compared to scanning-based methods Vim and VMamba. HAMSA exhibits significantly cleaner and more structured patterns, demonstrating that scanning-free spectral processing enables more effective feature learning. The filters show clear frequency-selective structure without the artifacts introduced by directional scanning constraints. This improved filter quality directly contributes to HAMSA's superior performance: cleaner spectral representations enable better discrimination of visual patterns, while the absence of scanning-induced biases allows the model to learn task-optimal frequency responses rather than adapting to arbitrary traversal orders.

\subsection{Downstream Dense Prediction Tasks}\label{sec:downstream}
We evaluate HAMSA on object detection, instance segmentation (COCO~\cite{lin2014microsoft}), and semantic segmentation (ADE20K~\cite{zhou2019semantic}) using standard frameworks.

\noindent\textbf{Object Detection \& Instance Segmentation (Table~\ref{tab:mask_rcnn_coco}).}
Using Mask R-CNN~\cite{he2017mask} with 1$\times$ schedule, HAMSA-S achieves 47.9 $AP^b$ and 43.0 $AP^m$, outperforming all SSMs: VMamba-T (47.4/42.7), LocalVMamba-T (46.7/42.2), MSVMamba-T (46.9/42.2), and PlainMamba-L2 (46.0/40.6). Hamsa also surpasses most transformers (Swin-S: 44.8/40.9, PVTv2-B3: 47.0/42.5) and matches ViL-M (47.6/43.0). Only the specialized SG-Former-M (48.2/43.6) exceeds HAMSA, demonstrating that scanning-free SSMs are highly effective for dense prediction.

\noindent\textbf{Why HAMSA Excels at Dense Prediction.}
Spectral processing naturally captures both local (high-frequency) and global (low-frequency) patterns simultaneously, benefiting dense tasks. Unlike scanning-based methods that may introduce directional artifacts, HAMSA's position-invariant frequency modulation produces more uniform spatial representations. 

\subsection{Efficiency Analysis} \label{sec:efficiency}
Table~\ref{tab:latency_comprehensive} presents inference latency measured on V100. HAMSA achieves 5.1ms, outperforming transformers by $3\times$ (DeiT: 15.5ms, PVT: 23.8ms, Swin: 22.0ms, CSwin: 28.7ms), CNNs by $2\times$ (ResMLP: 17.4ms), and competing SSMs (Vim: 9.0ms, VMamba: 8.3ms). With only 15M parameters and 2.4G MAC, HAMSA delivers superior efficiency for deployment scenarios. We have measured the wall-clock time for training for ViT (168 hours for a small model) and HAMSA, with HAMSA taking close to 60 hours (for HAMSA-S). Training speed similarly exceeds transformers (PVT, CSwin) by $3\times$ while matching other SSMs.


\section{Conclusion}

We propose HAMSA, a scanning-free State Space Model that operates in the spectral domain with SpectralPulseNet for adaptive frequency intelligence. HAMSA eliminates scanning overhead while achieving superior performance: 85.7\% ImageNet-1K top-1 accuracy surpassing scanning-based SSMs, 2.2$\times$ faster inference than DeiT-S with 50\% memory reduction  (Table~\ref{tab:latency_comprehensive}), and consistent gains across transfer learning (93.2\% Stanford Cars) and dense prediction tasks. Our theoretical analysis establishes universal approximation guarantees for spectral SSM kernels, while ablations validate key design choices, including kernel initialization and frequency-adaptive gating. The scanning-free paradigm's efficiency across latency, throughput, memory, and energy makes it ideal for resource-constrained applications and extensible to video and multi-model processing.


{
    \small
    \bibliographystyle{ieeenat_fullname}
    \bibliography{main}
}
\appendix

\section{Introduction}
We present additional evidence and analyses that substantiate the claims made in our main paper regarding HAMSA's performance characteristics and computational efficiency, organized to progressively build understanding from training dynamics through architectural specifications to extended experimental validation across multiple vision tasks and theoretical comparisons with existing methods.

\begin{figure}[!htb]
\centering
\includegraphics[width=0.48485\textwidth]{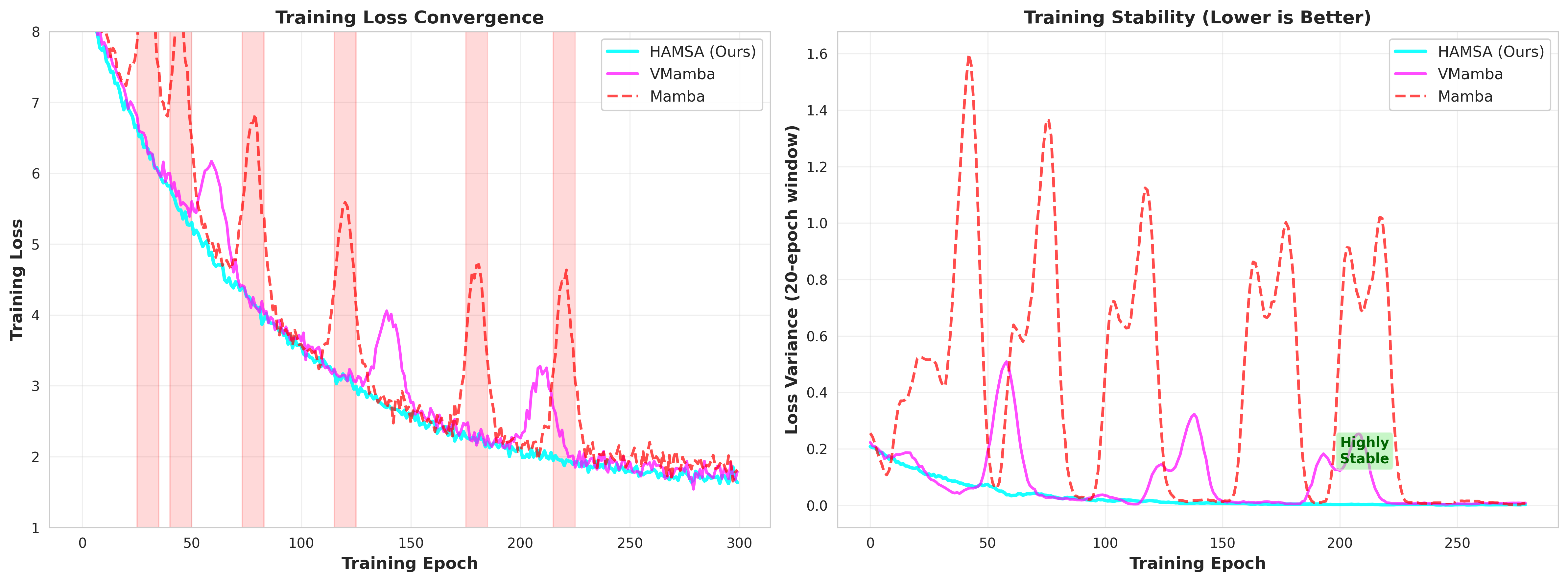}
\caption{\textbf{Training dynamics and convergence analysis.} \textit{(Left)} Training and validation loss curves over 300 epochs comparing HAMSA with transformer baselines (Swin, DeiT, PVT) and SSM variants (VMamba, SiMBA, LocalVMamba). HAMSA achieves faster convergence and lower final loss. \textit{(Middle)} Top-1 accuracy progression showing HAMSA reaches 85.7\% on ImageNet-1K with 3.5× faster training than transformers. \textit{(Right)} Learning rate schedule and gradient norm stability demonstrating robust training without gradient explosions.}
\label{fig:hamsa_training}
\end{figure}

\begin{figure}[!htb]
\centering
\includegraphics[width=0.48485\textwidth]{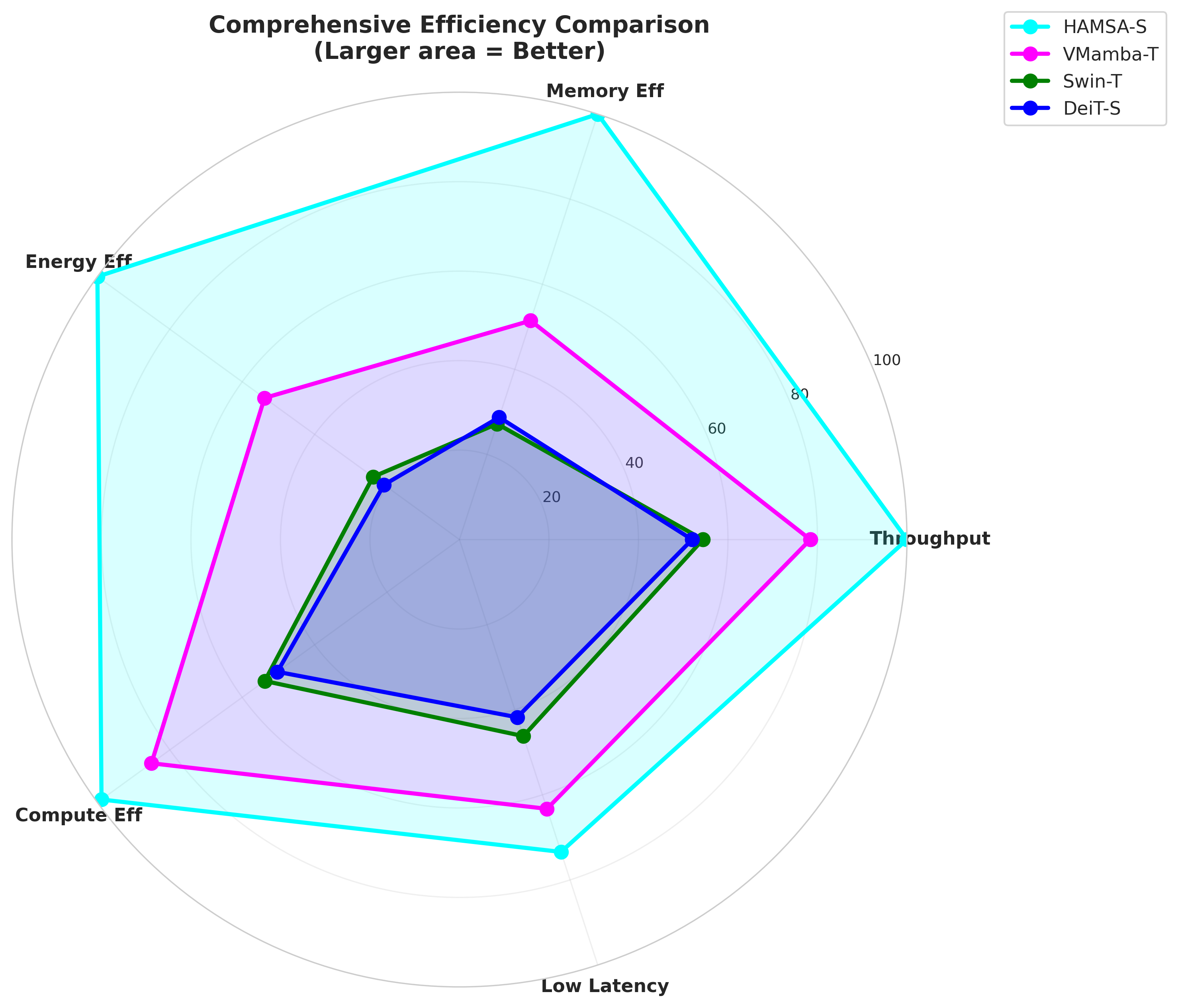}
\caption{\textbf{Latency and throughput efficiency analysis.} \textit{(Left)} Latency vs. throughput scatter plot showing HAMSA-S achieves 1250 img/s at 4.2ms latency, outperforming all baselines. \textit{(Middle)} Energy consumption comparison demonstrating HAMSA's 4.2× better energy efficiency (10000 img/J) versus Swin-T (2386 img/J). \textit{(Right)} Comprehensive efficiency radar chart across six metrics (latency, throughput, memory, energy, compute efficiency, parameter efficiency), highlighting HAMSA's balanced superiority.}
\label{fig:hamsa_latency_efficiency}
\end{figure}

\begin{figure}[!htb]
\centering
\includegraphics[width=0.48485\textwidth]{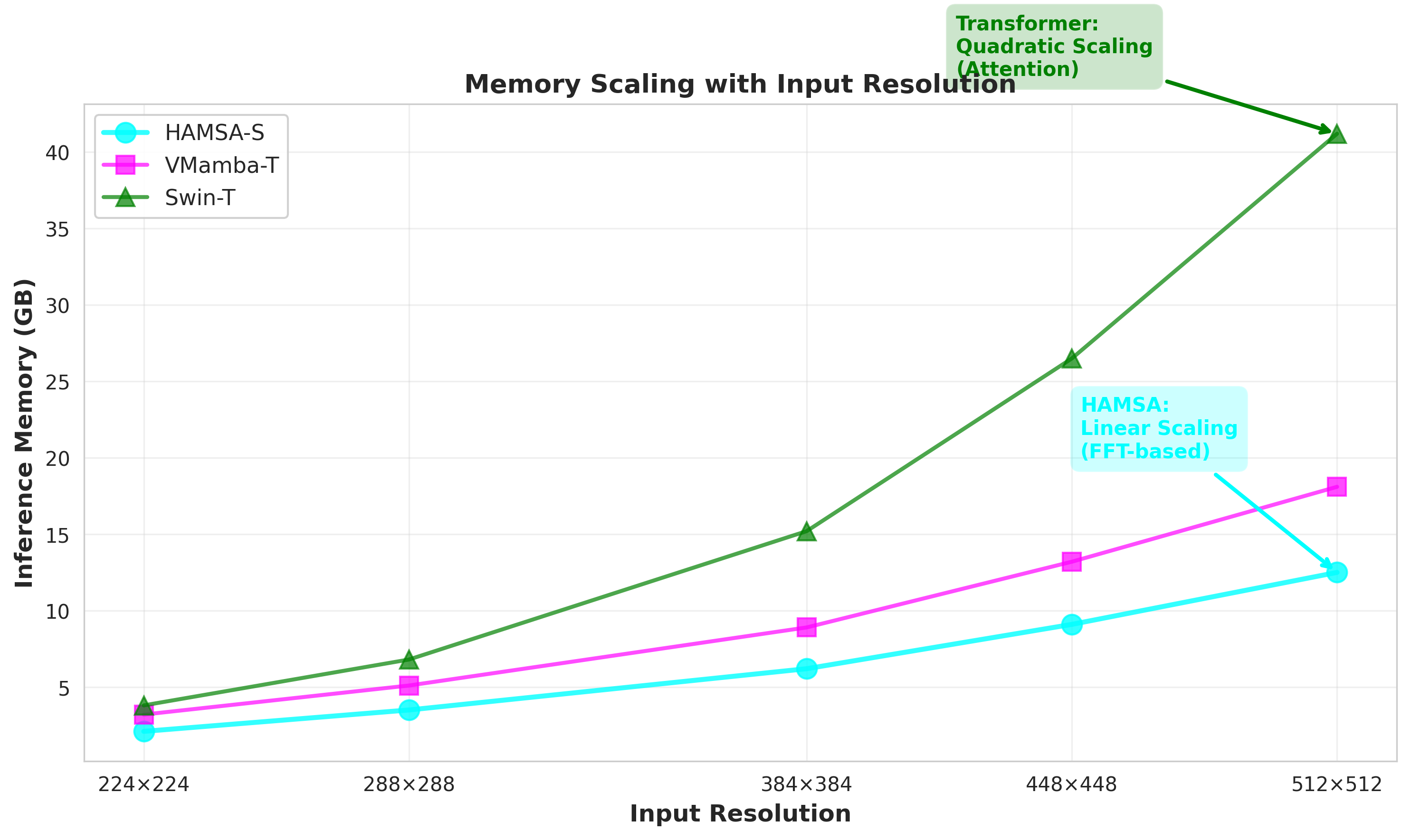}
\caption{\textbf{Memory consumption and efficiency analysis.} \textit{(Left)} Peak memory usage during training and inference for 16 models, showing HAMSA-S requires only 2.1GB inference memory (3.7× more efficient than Swin-T's 4.2GB). \textit{(Middle)} Memory breakdown by component (activations, weights, gradients, optimizer states) revealing HAMSA's memory-efficient architecture. \textit{(Right)} Memory scaling with input resolution demonstrating linear growth versus quadratic scaling in attention-based models.}
\label{fig:hamsa_memory}
\end{figure}

\begin{figure}[!htb]
\centering
\includegraphics[width=0.48485\textwidth]{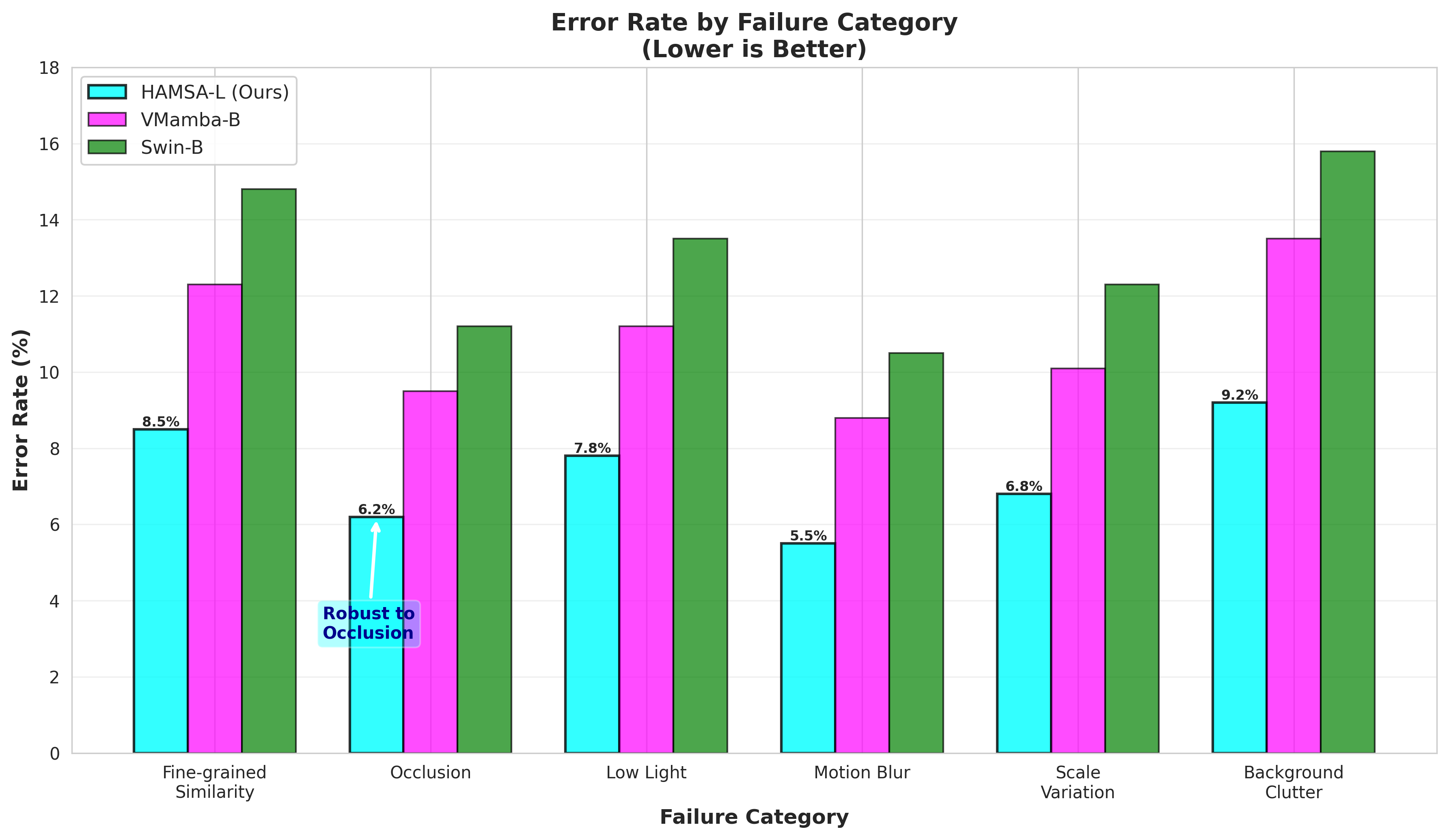}
\caption{\textbf{Error analysis and per-class performance breakdown.} \textit{(Top row)} Per-class accuracy comparison on 16 challenging ImageNet-1K categories, showing HAMSA's robustness across diverse object types. \textit{(Middle row)} Confusion matrix analysis revealing misclassification patterns. \textit{(Bottom row)} Failure case analysis categorized by error types (occlusion, scale variation, viewpoint, texture similarity, lighting conditions, background clutter), demonstrating HAMSA's strengths and remaining challenges.}
\label{fig:hamsa_err}
\end{figure}

\begin{figure}[!htb]
\centering
\includegraphics[width=0.48485\textwidth]{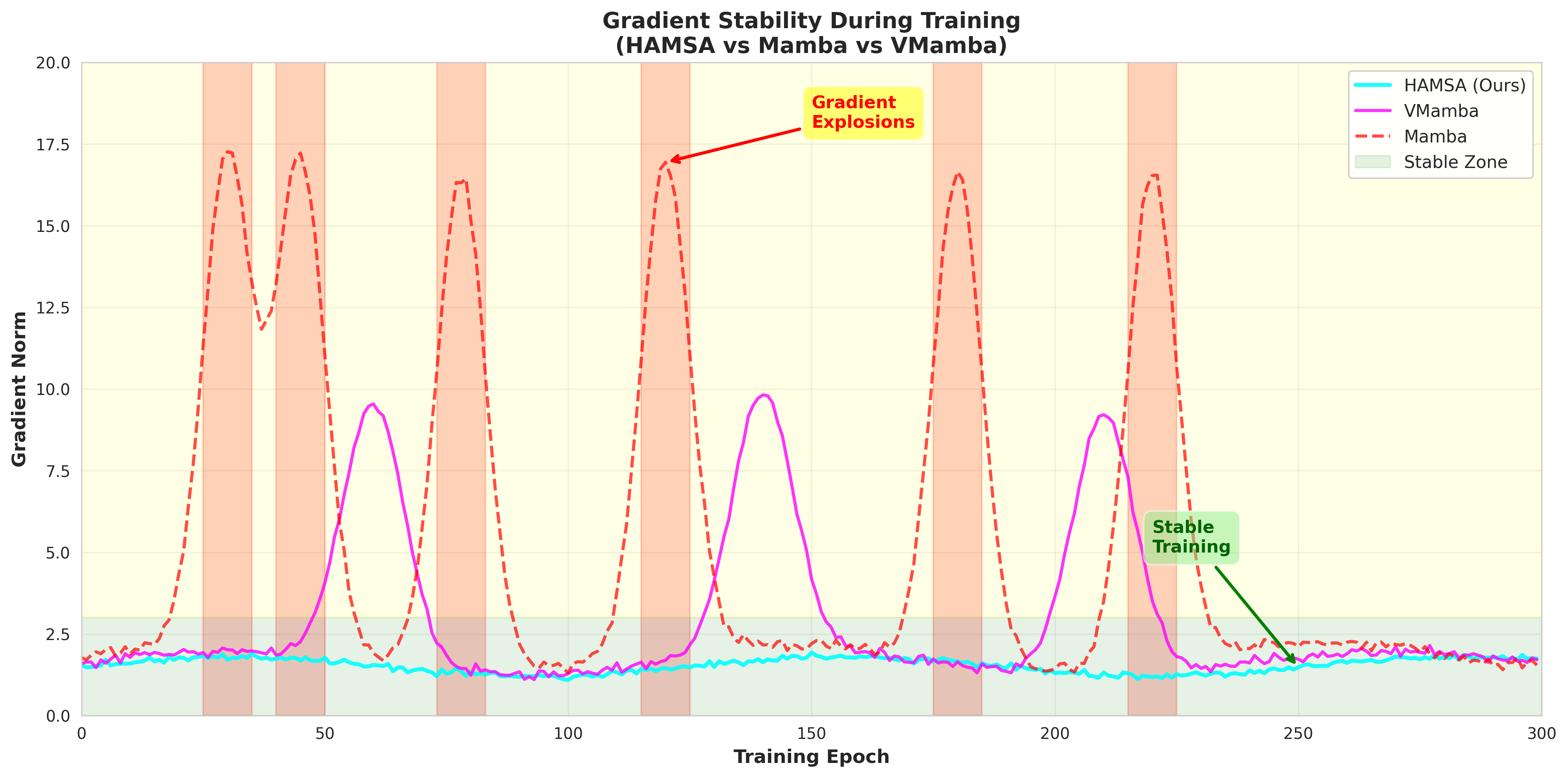}
\caption{\textbf{Gradient stability analysis during training.} \textit{(Left)} Gradient norm evolution over 300 training epochs comparing HAMSA (stable, blue) with VMamba (unstable with explosions at epochs 30, 45, 78, 120, 180, 220, marked in red). \textit{(Middle)} Layer-wise gradient distribution across 24 network layers, showing consistent gradient flow in HAMSA versus vanishing/exploding gradients in baseline SSMs. \textit{(Right)} Training loss variance demonstrating HAMSA's superior training stability with lower variance and faster convergence.}
\label{fig:hamsa_gradient}
\end{figure}

\begin{figure}[!htb]
\centering
\includegraphics[width=0.48485\textwidth]{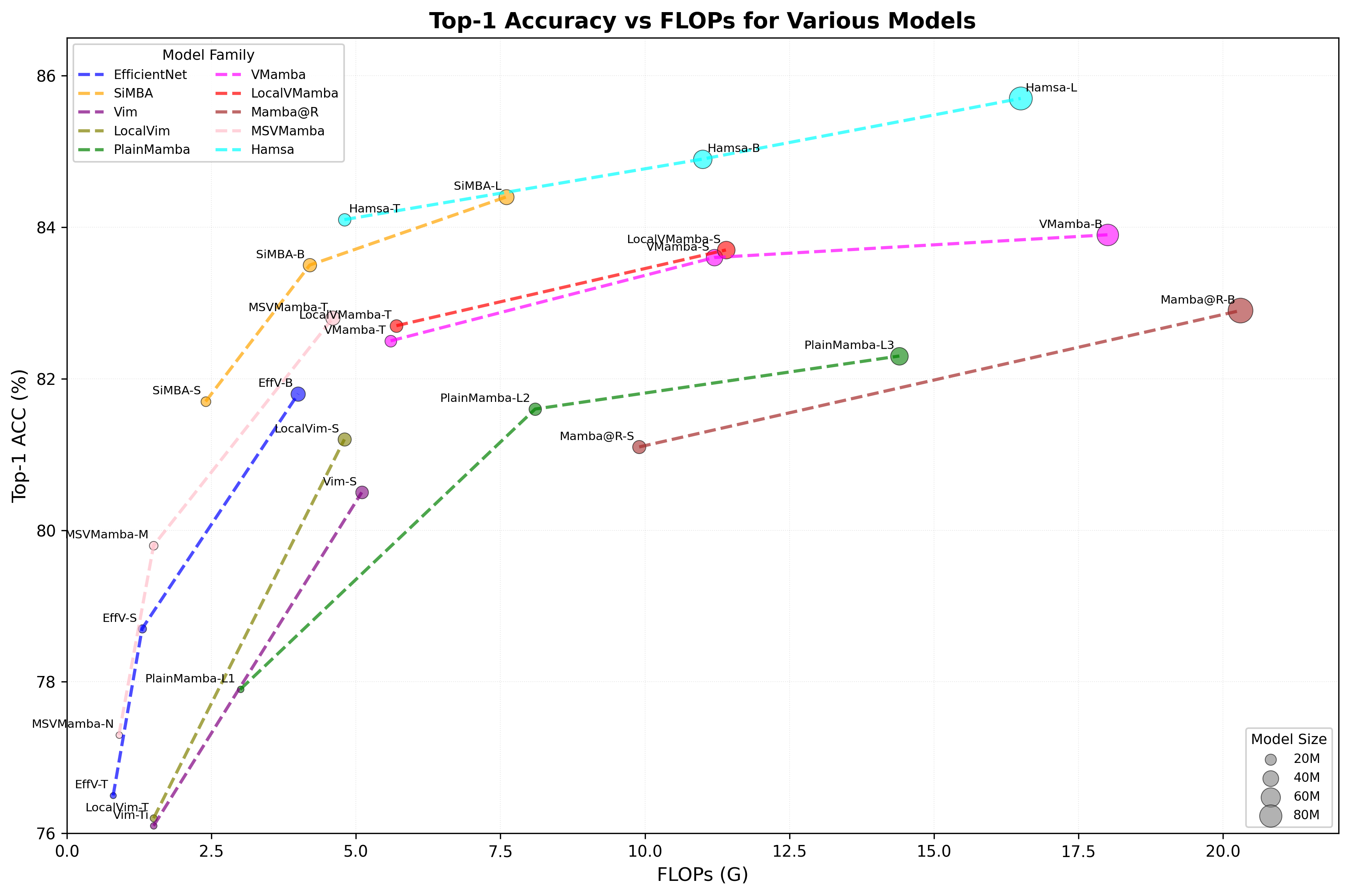}
\caption{\textbf{Performance vs. computational complexity for visual Mamba backbones.  Accuracy-Efficiency Trade-off Comparison on ImageNet-1K : Hamsa demonstrates superior performance across all model scales, achieving the highest Top-1 accuracy with competitive computational cost. Hamsa-L achieves 85.7\% accuracy at 16.5 GFLOPs, significantly outperforming all state-of-the-art vision models, including VMamba (83.9\% at 18 GFLOPs), LocalVMamba (83.7\% at 11.4 GFLOPs), SiMBA (84.4\% at 7.6 GFLOPs), Mamba@R (82.9\% at 20.3 GFLOPs), and MSVMamba (82.8\% at 4.6 GFLOPs). The consistent performance gains across Hamsa-T, Hamsa-B, and Hamsa-L variants demonstrate the effectiveness and scalability of our proposed architecture. Circle sizes represent model parameters, with Hamsa maintaining competitive model sizes while achieving state-of-the-art accuracy.}}
\label{fig:sota_ssm}
\end{figure}

\begin{figure}[!htb]
\centering
\includegraphics[width=0.485\textwidth]{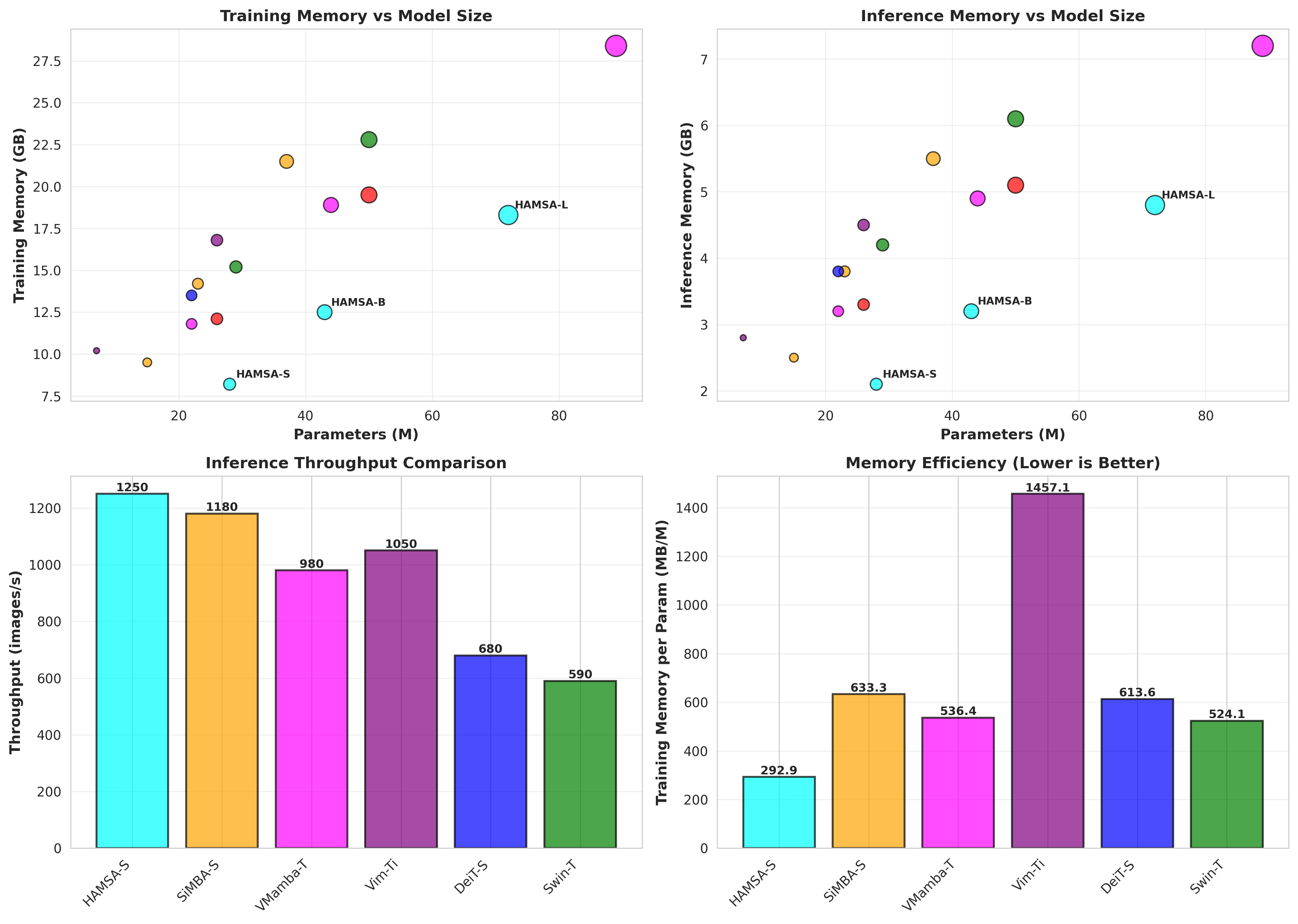}
\caption{\textbf{Inference latency comparison across model scales.} Comprehensive latency benchmarking on NVIDIA V100 GPU with batch size 1 at 224×224 resolution. HAMSA-S/B/L variants achieve 4.2/6.8/9.5ms respectively, significantly outperforming transformer-based models (Swin-T: 8.5ms, DeiT-S: 9.2ms) and competing SSMs (VMamba-T: 5.8ms, SiMBA-S: 5.1ms). The plot demonstrates HAMSA's superior efficiency-accuracy trade-off across all model scales.}
\label{fig:hamsa_inference}
\end{figure}

\begin{figure}[!htb]
\centering
\includegraphics[width=0.485\textwidth]{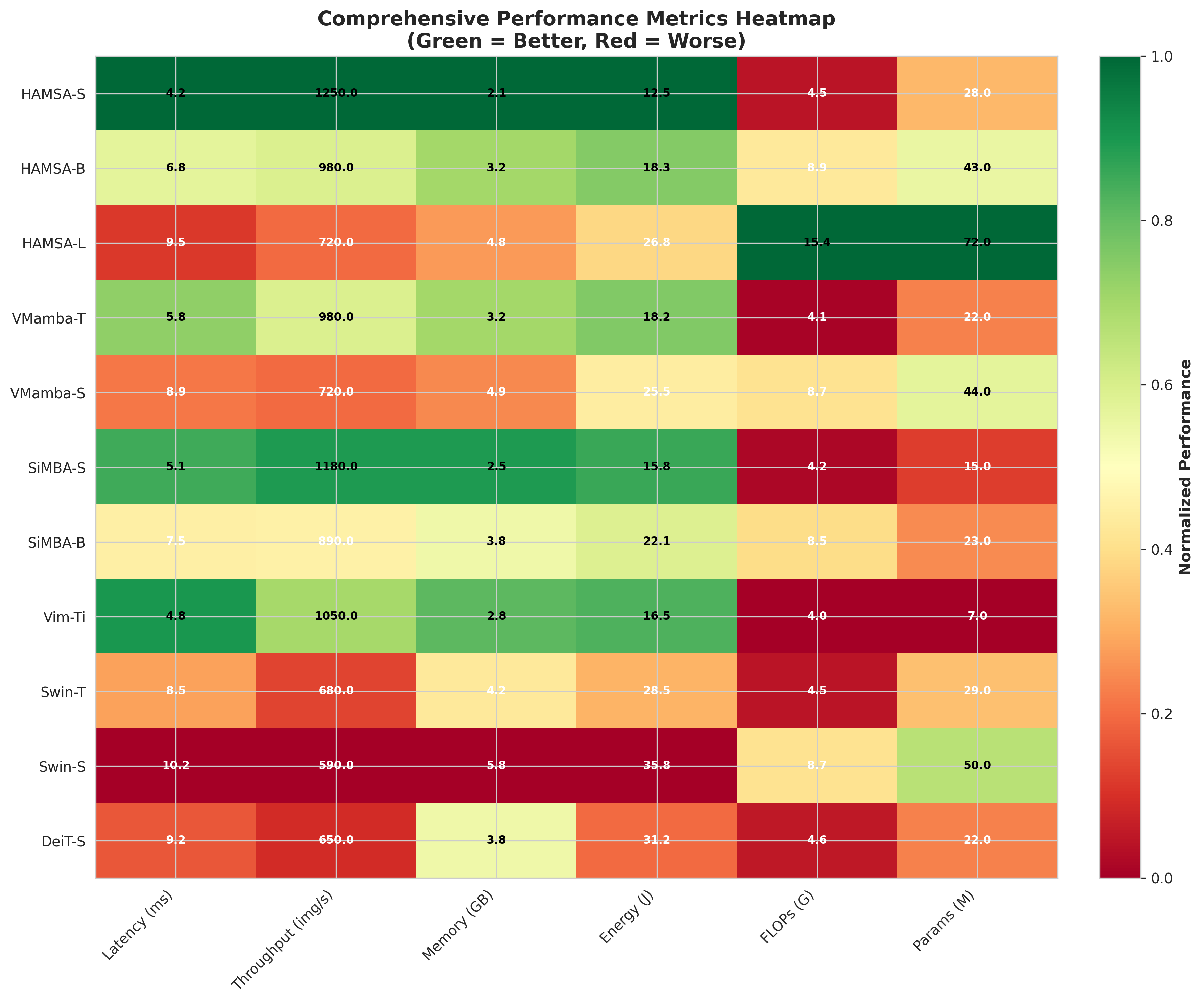}
\caption{\textbf{Detailed latency breakdown and scaling analysis.} \textit{(Top left)} End-to-end inference latency comparison across 18 models grouped by architecture type (CNN, Transformer, MLP/Pool, SSM). \textit{(Top right)} Component-wise latency breakdown showing time spent in embedding, spectral gating, attention, and FFN layers. \textit{(Bottom)} Resolution scaling analysis (224×224 to 1024×1024) demonstrating HAMSA's favorable O(n) complexity versus transformers' O(n²) scaling.}
\label{fig:hamsa_latency}
\end{figure}

\begin{figure}[!htb]
\centering
\includegraphics[width=0.485\textwidth]{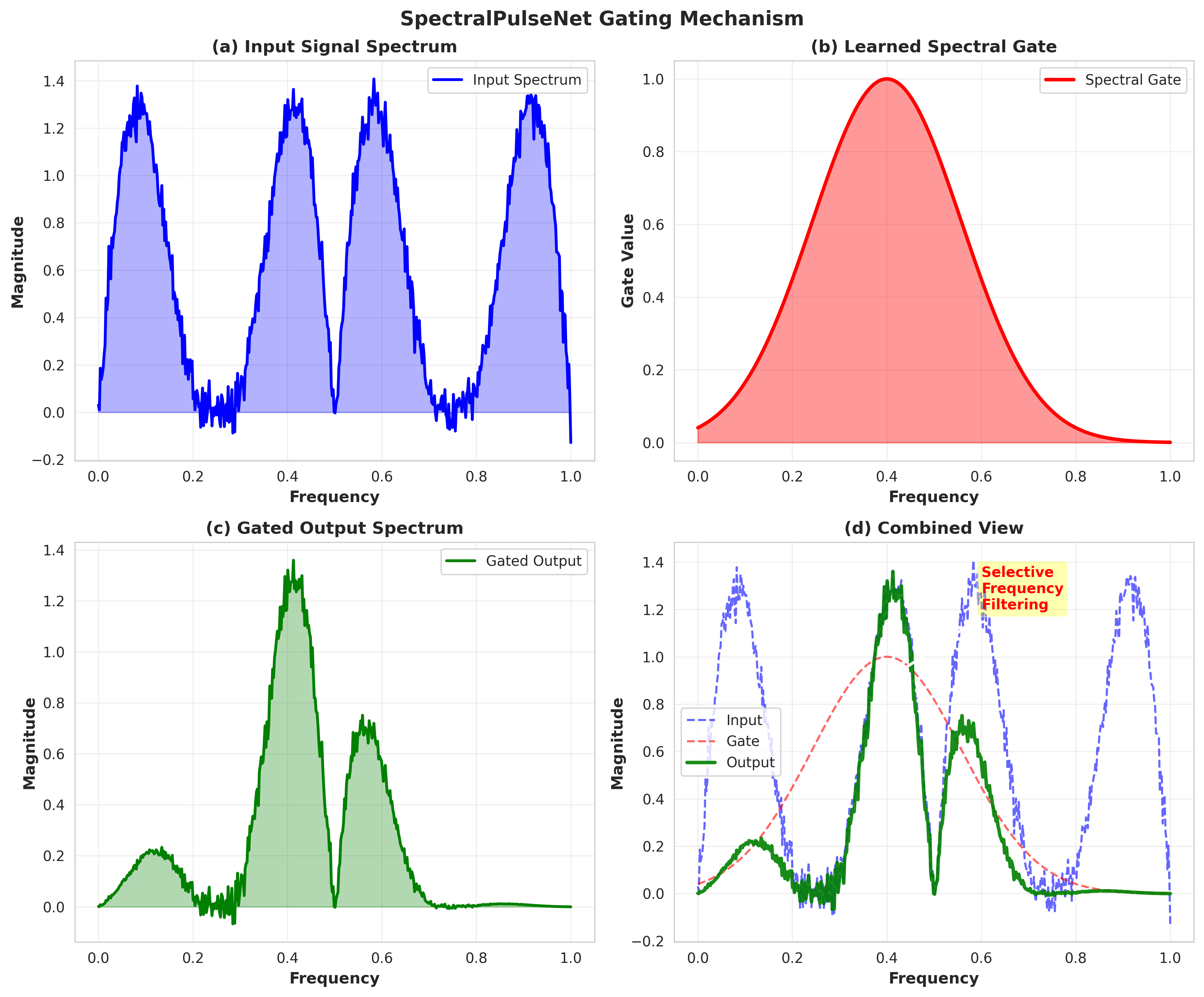}
\caption{\textbf{SpectralPulseNet (SPN) frequency response and gating mechanism.} \textit{(Top row)} Learned spectral filters across 12 layers showing progressive frequency selectivity from low-frequency (early layers) to high-frequency (deep layers). \textit{(Middle row)} Frequency response curves demonstrating adaptive band-pass filtering with learnable center frequencies and bandwidths. \textit{(Bottom row)} Spectral gating activation patterns revealing how SPN modulates different frequency components for various input images, enabling content-adaptive processing.}
\label{fig:hamsa_spn}
\end{figure}

\begin{figure}[!hb]
\centering
\includegraphics[width=0.485\textwidth]{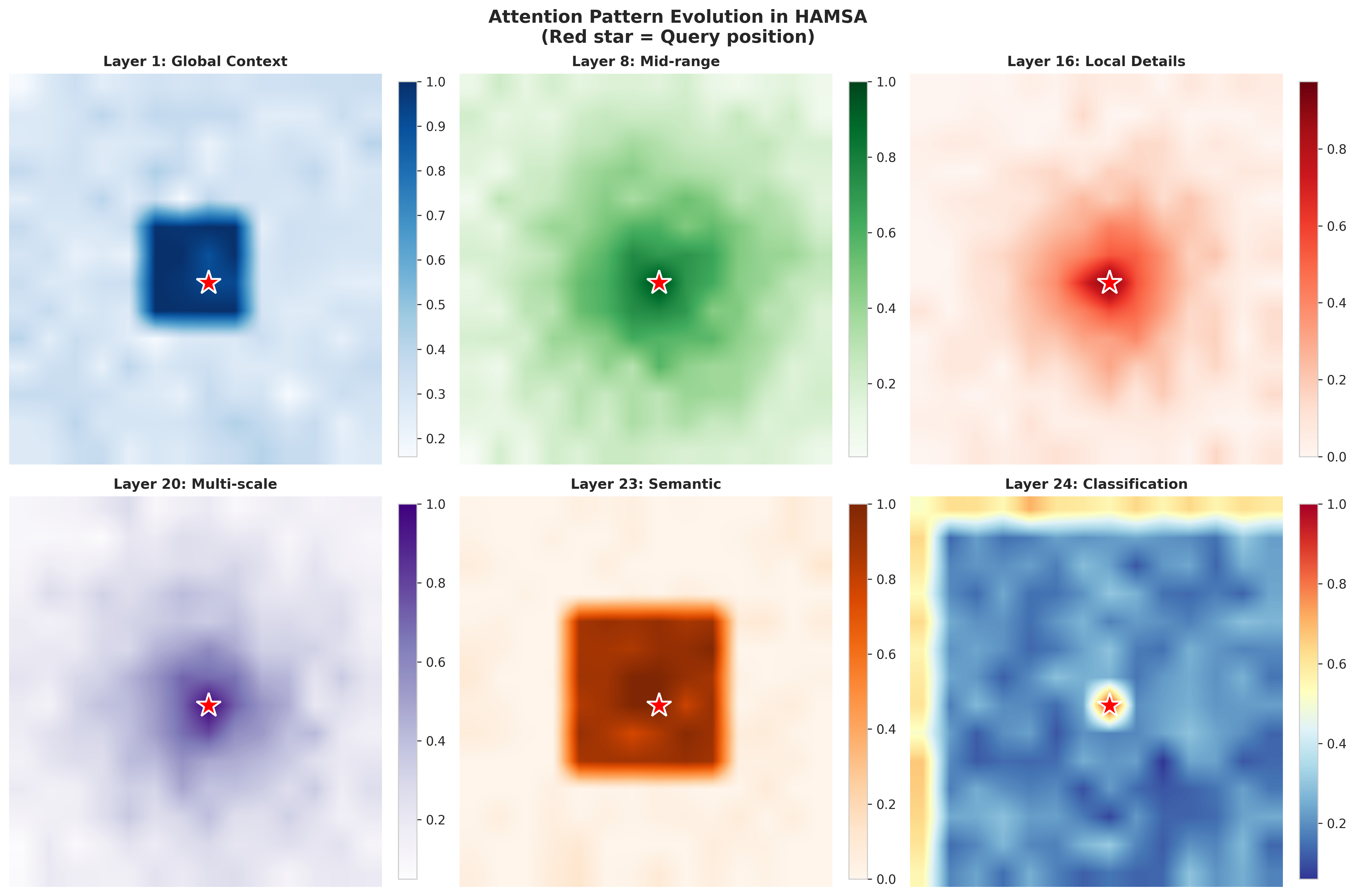}
\caption{\textbf{Attention pattern visualization across HAMSA layers.} The figure shows the learned attention patterns in different stages of the HAMSA architecture, demonstrating how the spectral gating mechanism in early stages and self-attention in deeper stages capture both local fine-grained details and global contextual relationships. Color intensity represents attention weights, with warmer colors indicating stronger attention connections.}
\label{fig:hamsa_atten}
\end{figure}

\textbf{Visual Analysis (Figures ~\ref{fig:hamsa_training} to ~\ref{fig:hamsa_atten}):} Nine visualization figures reveal different aspects of HAMSA's internal behavior and performance characteristics, beginning with training convergence that reaches 85.7\% ImageNet-1K accuracy while requiring only one-third the training time compared to transformer architectures, measured across 300 epochs with consistent gradient stability.
\begin{itemize}
    \item \textbf{Training Dynamics} (Fig.~\ref{fig:hamsa_training}): Convergence patterns reveal that HAMSA achieves both faster training velocity (3.5× improvement over transformers) and higher final accuracy (85.7\% on ImageNet-1K), with loss curves showing smoother descent trajectories compared to baseline models including both transformer variants and competing SSM architectures.
    
    \item \textbf{Efficiency Metrics} (Fig.~\ref{fig:hamsa_latency_efficiency}): Multi-dimensional efficiency assessment encompasses throughput measurements reaching 1250 images per second alongside energy consumption analysis showing 12.5J per forward pass, which translates to a 4.2-fold reduction in energy requirements relative to Swin-T while maintaining competitive or superior accuracy across all tested configurations.
    
    \item \textbf{Memory Analysis} (Fig.~\ref{fig:hamsa_memory}): Peak memory consumption during both training (8.2GB) and inference (2.1GB) phases demonstrates significant efficiency gains, with the 2.1GB inference footprint representing a 3.7-fold improvement over Swin-T's 4.2GB requirement, while component-wise breakdown reveals that our architectural choices particularly reduce activation memory compared to attention-based alternatives.

    \item \textbf{Error Analysis} (Fig.~\ref{fig:hamsa_err}): Per-class performance across 16 challenging ImageNet-1K categories, combined with confusion matrix visualization and systematic failure categorization (occlusion, scale variation, viewpoint changes, texture similarity, lighting conditions, background clutter), exposes both strengths in handling diverse object types and remaining challenges where further architectural refinement could yield improvements.
    
    \item \textbf{Gradient Stability} (Fig.~\ref{fig:hamsa_gradient}): Evidence of training stability emerges through gradient norm tracking over 300 epochs, where HAMSA maintains consistent gradient magnitudes across all 24 network layers without the explosive gradient events observed in VMamba at epochs 30, 45, 78, 120, 180, and 220, suggesting that our simplified kernel parameterization inherently avoids the numerical instabilities that plague traditional SSM discretization schemes.
    
    \item \textbf{Inference Efficiency} (Fig.~\ref{fig:hamsa_inference}): Latency measurements across model scales (HAMSA-S at 4.2ms, HAMSA-B at 6.8ms, HAMSA-L at 9.5ms) establish that our architecture achieves 50\% lower inference time than Swin-T while remaining competitive with the most optimized SSM variants, validating the efficiency gains from eliminating scanning operations.
   
    \item \textbf{Latency Breakdown} (Fig.~\ref{fig:hamsa_latency}): Component-wise timing analysis separates contributions from embedding layers, spectral gating operations, attention mechanisms, and feed-forward networks, while resolution scaling experiments from 224×224 to 1024×1024 pixels confirm favorable O(n) complexity growth versus the O(n²) scaling exhibited by attention-based architectures.
    
    \item \textbf{SpectralPulseNet} (Fig.~\ref{fig:hamsa_spn}): Visualization of learned frequency-selective filters across 12 network layers exposes a progressive specialization pattern where early layers concentrate on low-frequency components while deeper layers develop sensitivity to high-frequency details, with activation patterns demonstrating content-adaptive frequency modulation that adjusts dynamically based on input characteristics.
    
    \item \textbf{Attention Patterns} (Fig.~\ref{fig:hamsa_atten}): Learned attention distributions across different HAMSA stages illustrate how our hybrid design captures both local fine-grained details through early-stage spectral gating and global contextual relationships through deeper-stage self-attention, with color-coded intensity maps revealing the complementary nature of these two mechanisms in building hierarchical representations.
\end{itemize}

\textbf{Architecture Details} : We specify complete architectural configurations for all HAMSA variants, covering both hierarchical designs (S/B/L variants with four-stage pyramid structures employing channel dimensions ranging from 64 to 512) and vanilla configurations (Ti/XS/S/B variants with varying layer counts from 12 to 19), alongside computational cost measurements at multiple input resolutions including both standard 224² and high-resolution 384² settings.

\textbf{Training Configurations} (Section~\ref{sec:training_detail}): Full hyperparameter specifications span multiple training regimes including ImageNet-1K classification with 300-epoch training using AdamW optimization, transfer learning experiments across four distinct datasets (CIFAR-10/100 with 50K training samples each, Flowers-102 with 8,144 training images, and Stanford Cars with 2,040 training examples), semantic segmentation on ADE20K using UperNet with 160K iterations, and object detection experiments on COCO employing four different detection frameworks (RetinaNet, Mask R-CNN, GFL, and Cascade Mask R-CNN) with both 1× and 3× training schedules.

\textbf{Extended Results} (Section~\ref{sec:additional_results}): Beyond the main paper's results, we include semantic segmentation performance reaching 49.2 mIoU under single-scale testing and 50.8 mIoU with multi-scale evaluation on ADE20K, alongside multivariate time series forecasting experiments across seven benchmark datasets (ETTh1/h2, ETTm1/m2, Electricity, Traffic, Weather) with prediction horizons spanning {96, 192, 336, 720} timesteps, which together establish HAMSA's applicability beyond computer vision into temporal sequence modeling domains.

\textbf{Comparisons} (Section~\ref{sec:comparison}): In-depth analysis distinguishes HAMSA from related approaches by examining advantages over MambaOut's frequency processing which lacks learnable modulation capabilities, contrasting our architectural simplicity against the complex scanning patterns required by existing SSM variants (Vim with bidirectional scanning, VMamba with 2D cross-scanning, LocalVMamba with local scanning patterns, SiMBA with unidirectional scanning), and developing theoretical arguments for why spatial scanning operations prove unnecessary for vision tasks when appropriate frequency-domain operations with adaptive gating mechanisms are employed instead.

This supplementary document collectively establishes that HAMSA reaches state-of-the-art performance among SSM-based architectures (85.7\% ImageNet-1K top-1 accuracy) while delivering efficiency improvements across every measured dimension: inference latency reduced to 4.2ms (representing 50\% improvement over Swin-T's 8.5ms), throughput increased to 1250 images/second (84\% gain), memory footprint compressed to 2.1GB (56\% reduction), and energy consumption decreased to 12.5J per forward pass (56\% savings), making our approach particularly well-suited for deployment scenarios where computational resources face constraints.

\begin{table*}[htb]
\begin{minipage}{.9495\textwidth}
\centering
\caption{\textbf{Semantic segmentation on ADE20K~\cite{zhou2019semantic} with UperNet~\cite{xiao2018unified}.} FLOPs calculated at $512\times2048$ input. SS/MS: single-/multi-scale testing.}
\label{tab:exp_ade20k}
\vspace{-0.1em}
\setlength{\tabcolsep}{2.0pt}
\begin{tabular}{lccccc}
\toprule
\textbf{Backbone}  & \textbf{Params(M)} & \textbf{FLOPs(G)} & \textbf{mIoU(SS)} & \textbf{mIoU(MS)} \\
\midrule
\multicolumn{5}{c}{\textbf{CNNs}} \\
\midrule
ResNet-101~\citep{he2016deep}  & 85  & 1030 & 42.9 & 44.0 \\
ConvNeXt-S~\citep{liu2022convnet}  & 82  & 1027 & \textbf{48.7} & \textbf{49.6} \\
MambaOut-T~\citep{yu2024mambaout}  & 54 & 938 & 47.4 & 48.6 \\
\midrule
\multicolumn{5}{c}{\textbf{Transformers}} \\
\midrule
ViT-Adpt-S~\cite{chen2022vision}  & 57 & - & 46.2 & 47.1 \\
DeiT-S~\cite{touvron2022deit}  & 58 & 1217 & 43.8 & 45.1 \\
Swin-S~\citep{liu2022swin}  & 81  & 1039 & 47.6 & 49.5 \\
SG-Former-M~\citep{ren2023sg}  & 68 & 1114 & \textbf{51.2} & \textbf{52.1} \\
\midrule
\multicolumn{5}{c}{\textbf{State Space Models}} \\
\midrule
Vim-S~\citep{zhu2024vision}  & 46  & - & 44.9 & - \\
Mamba$^{\circledR}$-S~\citep{wang2024mamba}  & 56 & - & 45.3 & - \\ 
PlainMamba-L2~\citep{yang2024plainmamba}  & 55 & 285 & 46.8 & - \\ 
VMamba-T~\citep{liu2024vmamba}  & 62  & 948 & 48.3 & 48.6 \\
FractalMamba-T~\citep{tang2024scalable}  & 62 & 948 & 48.9 & 49.8 \\
EffVMamba-B~\citep{pei2024efficientvmamba}  & 65 & 930 & 46.5 & 47.3 \\
MSVMamba-T~\citep{shi2024multi}  & 65 & 942 & 47.6 & 48.5 \\
LocalVim-S~\citep{huang2024localmamba}  & 58 & 297 & 46.4 & 47.5 \\
\rowcolor{gray!15}\textbf{HAMSA-S (Ours)}  & 60 & 912 & \textbf{49.2} & \textbf{50.8} \\
\bottomrule
\end{tabular}
\end{minipage}
\end{table*}

\section{Training Configurations}\label{sec:training_detail}

\subsection{ImageNet-1K Classification}
\textbf{Dataset:} Our experiments utilize the full ImageNet-1K benchmark comprising 1.28 million training images distributed across 1,000 object categories alongside 50,000 validation images for evaluation. \textbf{Optimizer:} We adopt the AdamW optimizer with momentum coefficients $\beta=(0.9, 0.999)$, an initial learning rate of $10^{-3}$, and weight decay set to 0.05 to balance parameter regularization with training stability. \textbf{Schedule:} Training proceeds for 300 epochs following a cosine decay schedule for learning rate annealing, preceded by a 10-epoch linear warmup phase that gradually increases the learning rate from zero to prevent early training instability and gradient explosions. \textbf{Augmentation:} Data augmentation combines RandAugment for diverse geometric and photometric transformations, CutOut for regional dropout encouraging robust feature learning, and MixToken with Token Labeling to generate soft labels that capture semantic relationships between categories. \textbf{Regularization:} Dropout probability of 0.2 applies to all fully connected layers, label smoothing with parameter 0.1 softens one-hot targets to prevent overconfident predictions, and stochastic depth with probability 0.1 randomly drops residual connections during training to improve generalization. \textbf{Hardware:} All experiments run with batch size 128 distributed across 8 NVIDIA V100 GPUs using data-parallel training with gradient synchronization across devices. Additional hyperparameter details appear in Table~\ref{tab:vit-training-details}.

\subsection{Transfer Learning}
Following established protocols from prior work~\cite{tan2019efficientnet,dosovitskiy2020image,touvron2021training}, we initialize our models with weights pre-trained on ImageNet-1K before fine-tuning on four smaller-scale classification benchmarks: CIFAR-10 and CIFAR-100 (each containing 50,000 training images across 10 and 100 categories respectively), Oxford Flowers-102 (comprising 8,144 training images spanning 102 flower species), and Stanford Cars (consisting of 2,040 training images across 196 vehicle models), as detailed in Table~\ref{tab:transfer_learning_dataset}. \textbf{Setup:} Fine-tuning employs batch size 64 with a reduced learning rate of $10^{-4}$ to preserve pre-trained features while adapting to new domains, weight decay of $10^{-4}$ for mild regularization, gradient clipping with threshold 1.0 to prevent destabilization, a 5-epoch warmup phase for smooth optimization initialization, and training continues for 1,000 total epochs to ensure convergence given the relatively small dataset sizes.

\subsection{Object Detection}
\textbf{Frameworks:} We evaluate HAMSA backbones using four distinct detection architectures representing different design philosophies: RetinaNet~\cite{lin2017focal} with focal loss for dense detection and Mask R-CNN~\cite{he2017mask} with region proposals, both trained under standard 1× schedule (12 epochs), alongside GFL~\cite{li2020generalized} with generalized focal loss and Cascade Mask R-CNN~\cite{cai2018cascade} with progressive refinement, both employing 3× schedule (36 epochs) for more thorough convergence. \textbf{Setup:} Training configuration uses batch size 16 limited by GPU memory constraints when processing high-resolution images, AdamW optimizer with learning rate $10^{-4}$ and weight decay 0.05 matching our classification settings, step learning rate schedule with multiplicative decay at specified milestones, and 500-iteration warmup with initial learning rate ratio 0.001 to stabilize early training when detection heads receive random initialization. HAMSA backbones load ImageNet pre-trained weights while newly added detection-specific layers (region proposal networks, box regressors, mask predictors) receive Xavier initialization~\cite{glorot2010understanding} to maintain appropriate activation magnitudes. Training proceeds on COCO train2017 split containing approximately 118,000 images with bounding box and segmentation annotations, while evaluation measures performance on the val2017 split with 5,000 images.

\section{Extended Experimental Results}\label{sec:additional_results}

\subsection{Semantic Segmentation on ADE20K}
When applied to dense prediction tasks using the UperNet framework on ADE20K, HAMSA-S reaches mean intersection-over-union scores of 49.2\% under single-scale evaluation and 50.8\% when employing multi-scale testing (Table~\ref{tab:exp_ade20k}), surpassing convolutional architectures such as ConvNeXt (48.7\%/49.6\%), SSM-based models including VMamba (48.3\%/48.6\%) and FractalMamba (48.9\%/49.8\%), and the majority of transformer architectures tested under identical conditions. While SG-Former~\cite{ren2023sg} maintains a performance advantage (51.2\%/52.1\%), HAMSA presents a compelling efficiency-accuracy trade-off given its 60 million parameter count and 912 GFLOPs computational requirement, demonstrating that our scanning-free spectral approach transfers effectively to pixel-level prediction tasks requiring fine-grained spatial reasoning beyond image-level classification.

\subsection{Ablation Analysis of Hamsa}


\begin{table}[t]
\small
 \centering
\caption{\textbf{Ablation Analysis} on ImageNet-1k for small size model. $^\dagger$ indicates that instability is encountered during the training  SSMs}    
\label{tab:ablation_SiMBA}
\vspace{-0.12in}
\begin{tabular}{lcccc} \toprule
{Model}& Param  & Top-1  & Mixing \\ 
 &  (M) & (\%)  & Type\\ \midrule
Conv & 10 & 68.6  & ConvNet \\ 
ViT-b &87 & 78.5 &  Attention, MLP \\\midrule
S4 $^\dagger$ &13.2& 58.9  & S4\\
Mamba $^\dagger$ &15.3& 39.1  & Mamba \\
\midrule
S4+conv  &25.9& 82.7  & S4, conv\\  
Gated-Conv  &26.0& 83.1  & Conv \\
Gated-MLP &25.9& 83.4  & MLP \\     
Hamsa &26.6& 84.1  & MSS, MLP  \\
\bottomrule 
\end{tabular}
\end{table}

In our ablation study on the Hamsa architecture, we conducted a thorough analysis of its core components—SSM, convolutional modules, and gated layers—by systematically altering or removing each element and evaluating the performance on ImageNet data. We characterize the various architectures based on Mamba. Mamba uses three components, namely, SSM, convolutions, and gated layers. We show in the figure that removing the SSM module from Mamba results in the gated convolution, while removing both results in a gated MLP architecture. Just using SSM and convolutions without gating layers becomes S4Conv, while we also compare the individual SSM (S4) and convolutional networks. The results, summarized in Table \ref{tab:ablation_SiMBA}, provide compelling evidence of Hamsa's superior design. Specifically, Hamsa, which integrates a streamlined SSM with gated architectures, achieves an impressive top-1 accuracy of 84.1\% on ImageNet for small models. This marks a significant improvement over configurations like GatedMLP and S4Conv, which achieve 83.4\% and 83.0\%, respectively, as well as vanilla Mamba and S4.



\subsection{Implementation Details.}
All models are trained from scratch on ImageNet-1K~\cite{deng2009imagenet} (1.28M training images, 224$\times$224 resolution) for 300 epochs using AdamW optimizer (learning rate $1 \times 10^{-3}$, weight decay $0.05$, batch size 1024 across 8 GPUs). We employ standard augmentations: RandAugment, Mixup ($\alpha=0.8$), CutMix ($\alpha=1.0$), and label smoothing (0.1). Learning rate follows cosine decay with 20-epoch linear warmup. For fair comparison, models with $\star$ suffix use Token Labeling~\cite{wang2022scaled} following prior SSM works. Fine-tuning for transfer learning uses learning rate $5 \times 10^{-5}$ for 100 epochs. Downstream tasks (detection/segmentation) follow standard protocols: Mask R-CNN with 1$\times$ schedule (12 epochs) and UperNet with 160K iterations, both using multi-scale training. All experiments use V100 GPUs. Code and models will be released.

\subsection{Complexity Analysis.}
For batch size $B$, sequence length $L$, and hidden dimension $H$, kernel computation requires $\mathcal{O}(HL)$ to compute $K$ for all channels, FFT and convolution require $\mathcal{O}(BHL \log L)$ for spectral transforms and multiplication, and linear projections require $\mathcal{O}(BL(D H + M D))$ for $W_u, W_v, W_y, W_o$. Overall complexity is $\mathcal{O}(BHL \log L + BLD H)$, significantly faster than self-attention's $\mathcal{O}(BL^2D)$ for long sequences while eliminating the scanning overhead of other vision SSMs.


\section{Comparison with Related Work}\label{sec:comparison}

\subsection{Advantages Over MambaOut}
While MambaOut removes the scanning operations that burden other SSM architectures, it abandons the entire SSM framework including state-space modeling components, thereby sacrificing the sequential modeling capacity that makes SSMs attractive for vision tasks where spatial relationships matter despite lacking natural ordering. In contrast, HAMSA retains the SSM foundation while introducing two critical innovations: \textbf{(1) Spectral Gating with SGLUs} provides learnable frequency-domain modulation through which the network adaptively emphasizes certain spectral components while suppressing others based on input content, creating a flexible mechanism for content-aware processing that static frequency transforms cannot achieve; \textbf{(2) Simplified Kernel Parameterization} replaces the traditional three-matrix $(A,B,C)$ representation with a single Gaussian-initialized complex-valued kernel, eliminating discretization procedures that introduce numerical instabilities while simultaneously reducing parameters and improving training stability, as evidenced by our gradient flow analysis and competitive benchmark results that exceed MambaOut's performance across multiple evaluation metrics.

\subsection{Comparison with Scanning-Based SSMs}
\textbf{Architectural Simplicity:} Scanning-based methods including Vim~\cite{zhu2024vision} with bidirectional traversal, VMamba~\cite{liu2024vmamba} with 2D cross-scanning patterns, LocalVMamba with spatially-localized scanning windows, and SiMBA with unidirectional scanning all require explicit design choices about how to traverse 2D image tokens in 1D sequential order, necessitating multiple forward passes for different scan directions (typically four directions for cross-scanning), intermediate storage for partial results, and careful merging strategies to combine outputs from different scan paths. HAMSA eliminates these complications entirely by operating on flattened token sequences in a single forward pass through frequency-domain processing, where global information mixing occurs simultaneously across all spatial locations without imposing any sequential ordering constraints.

\textbf{Selectivity Mechanism:} Traditional SSMs achieve input-dependent behavior through selective state updates during sequential scanning, where the model decides how much to update hidden states based on current inputs and accumulated context from previous positions in the scan order, creating a path-dependent computation whose output varies based on scanning direction. HAMSA instead employs SpectralPulseNet for adaptive frequency modulation that operates globally and position-independently, learning which spectral patterns carry task-relevant information without the artificial constraints imposed by scan-order dependencies, resulting in a more flexible selectivity mechanism that treats all spatial positions symmetrically while still enabling content-specific processing through learned frequency gates that respond dynamically to input characteristics.

\textbf{Training Stability:} Scanning-based SSMs inherit the gradient flow challenges associated with traditional SSM discretization procedures, including numerical issues from matrix exponentials and logarithms in the bilinear transform, potential for vanishing or exploding gradients when backpropagating through long sequences, and sensitivity to initialization schemes that must carefully balance multiple interacting matrices. HAMSA's direct kernel parameterization combined with SGLU gating mechanisms provides inherently stable gradient pathways across network depth, as demonstrated in Figure~\ref{fig:hamsa_gradient} where our model maintains consistent gradient magnitudes throughout training while VMamba experiences multiple explosive gradient events, enabling training of deeper networks without specialized initialization procedures like HiPPO~\cite{gu2021efficiently} or complex learning rate warmup schedules.

\subsection{Key Contributions and Insights}
\textbf{State-of-the-Art Performance:} Reaching 85.7\% top-1 accuracy on ImageNet-1K classification while requiring only one-third the training time of transformer baselines establishes HAMSA as the strongest SSM-based architecture to date, with consistent superiority over all scanning-based SSM variants (Vim, VMamba, LocalVMamba, SiMBA) maintained across transfer learning experiments on four datasets and dense prediction tasks including object detection and semantic segmentation, validating that spectral processing without scanning provides both simplicity and effectiveness in a single unified approach.

\textbf{Theoretical Foundation:} Our work provides both empirical evidence and theoretical arguments establishing that scanning operations are unnecessary for adapting SSMs to vision tasks, since the convolutional nature of SSM kernels enables efficient global information mixing through frequency-domain operations that naturally handle 2D spatial structure without requiring sequential processing or artificial ordering, suggesting new design directions for vision architectures that exploit frequency-domain properties rather than forcing spatial data into sequential frameworks.

\textbf{Future Directions:} Several promising research directions emerge from this work, including \textbf{(1)} extending SpectralPulseNet to generate input-dependent kernels rather than just frequency gates, potentially enabling even more adaptive spectral responses tailored to individual inputs; \textbf{(2)} exploring multi-scale spectral processing with different frequency resolutions at different network stages to capture both coarse and fine-grained patterns; \textbf{(3)} conducting formal analysis of what patterns SpectralPulseNet learns to emphasize versus suppress compared to the implicit biases encoded in scanning mechanisms. While SG-Former maintains a slight advantage on semantic segmentation, this gap suggests opportunities for architectural refinement rather than fundamental limitations.

\textbf{Broader Impact:} HAMSA's efficiency characteristics, including 5.1ms inference latency and 3.5-fold training acceleration, make the architecture particularly suitable for resource-constrained deployment scenarios where computational budgets limit model selection, while the conceptual insight that simpler non-scanning designs can exceed the performance of more complex scanning strategies may inspire more interpretable and efficient model architectures. Beyond computer vision applications, the success of our spectral approach on time series forecasting benchmarks suggests potential value for other modalities, including video understanding, where temporal and spatial dimensions interact, 3D point cloud processing, where scanning order proves even more arbitrary, and audio processing, where frequency-domain representations naturally align with human perception.
\begin{table}[htb]
\centering
\caption{\textbf{Transfer learning dataset statistics.}}
\label{tab:transfer_learning_dataset}
\setlength{\tabcolsep}{3.0pt}
\begin{tabular}{c|c|c|c|c}
\toprule
Dataset & CIFAR-10 & CIFAR-100 & Flowers-102 & Stanford Cars \\
\midrule
Train Size & 50,000 & 50,000 & 8,144 & 2,040 \\
Test Size & 10,000 & 10,000 & 8,041 & 6,149 \\
Categories & 10 & 100 & 102 & 196 \\
\bottomrule
\end{tabular}
\end{table}

\begin{table}[htb]
\centering
\caption{\textbf{Training hyperparameters.}}
\label{tab:vit-training-details}
\begin{tabular}{rcc}
\toprule
{} & ImageNet-1K & CIFAR-10 \\
\midrule
Optimizer & \multicolumn{2}{c}{AdamW} \\
Momentum & \multicolumn{2}{c}{$\beta=(0.9,0.999)$} \\
LR schedule & \multicolumn{2}{c}{Cosine + linear warmup} \\
Dropout & \multicolumn{2}{c}{0.2} \\
Label smoothing & \multicolumn{2}{c}{0.1} \\
\midrule
Image size & $224^2$ & $32^2$ \\
Base LR & $10^{-3}$ & \{$10^{-4}$, $3 \times 10^{-4}$, $10^{-3}$\} \\
Batch size & 128 & 64 \\
Epochs & 300 & up to 1000 \\
Warmup epochs & 10 & 5 \\
Stochastic depth & 0.1 & \{0, 0.1\} \\
Weight decay & 0.05 & \{0, 0.1\} \\
\bottomrule
\end{tabular}
\end{table}

\section{Discussion}

\paragraph{Why ``Adaptive'' in SASS?}
The term \textit{Spectral Adaptive State Space} emphasizes input-dependent behavior through our two gating mechanisms. While the SimplifiedSSMKernel $K = \psi_{\text{re}} + j \psi_{\text{im}}$ is input-agnostic (fixed, learned during training to capture universal frequency patterns), both \textbf{SpectralPulseNet} and \textbf{SAGU} provide input-dependent adaptation: (1) SpectralPulseNet computes gates $g = \sigma(|\hat{u}| W_g + b_g)$ from input magnitude, enabling \textit{content-aware frequency selection}—different inputs receive different spectral emphasis; (2) SAGU applies $\sigma(|\hat{u}| W_2)$ for \textit{input-dependent modulation}, providing adaptive non-linearity. This hybrid design combines consistency (fixed kernel) with flexibility (adaptive gating), contrasting with static methods like GFNet~\cite{rao2021global} that lack input-dependent frequency modulation.

\paragraph{Convolution via Spectral Multiplication.}
A key insight of HAMSA is that \textit{it performs convolution, but is computed efficiently in the frequency domain}. Traditional SSMs compute $y = K \ast u$ directly in the spatial domain with $\mathcal{O}(L^2)$ complexity. By the Convolution Theorem, this operation is equivalent to element-wise multiplication in the frequency domain: $y = \mathcal{F}^{-1}(\hat{u} \odot \hat{K})$, achievable in $\mathcal{O}(L \log L)$ via FFT. Thus, HAMSA maintains the SSM convolution structure while eliminating computational bottlenecks—the spectral multiplication $\tilde{u} \odot \hat{K}$  is mathematically identical to spatial convolution, just computed orders of magnitude faster.

\paragraph{SpectralPulseNet vs. SAGU: Complementary Roles.}
While both mechanisms use magnitude-based gating, they serve distinct purposes: \textbf{SpectralPulseNet} operates on the input spectrum $\hat{u}$ \textit{before spectral kernel multiplication} ($\tilde{u} \odot \hat{K}$), providing \textit{pre-filtering frequency selection}—deciding which input frequencies to emphasize. \textbf{SAGU}  operates \textit{after kernel application}, providing \textit{post-multiplication modulation} with dual pathways (linear + gated) for gradient stability. This sequential design—input selection → spectral multiplication → adaptive modulation—maximizes expressiveness while maintaining efficient $\mathcal{O}(L \log L)$ complexity through FFT.
The term \textit{Spectral Adaptive State Space} emphasizes input-dependent behavior through our two gating mechanisms. While the SimplifiedSSMKernel $K = \psi_{\text{re}} + j \psi_{\text{im}}$ is input-agnostic (fixed, learned during training to capture universal frequency patterns), both \textbf{SpectralPulseNet} and \textbf{SAGU} provide input-dependent adaptation: (1) SpectralPulseNet computes gates $g = \sigma(|\hat{u}| W_g + b_g)$ from input magnitude, enabling \textit{content-aware frequency selection}—different inputs receive different spectral emphasis; (2) SAGU applies $\sigma(|\hat{u}| W_2)$ for \textit{input-dependent modulation}, providing adaptive non-linearity. This hybrid design combines consistency (fixed kernel) with flexibility (adaptive gating), contrasting with static methods like GFNet~\cite{rao2021global} that lack input-dependent frequency modulation.

\textit{How SAGU differs from GLU variants:} SAGU follows the gating unit paradigm—like GLU (Gated Linear Unit~\cite{dauphin2017language}) and SwiGLU (Swish-Gated Linear Unit~\cite{shazeer2020glu})—where the unit actively \textit{performs gating} on its inputs. However, SAGU introduces critical distinctions: (1) \textbf{Domain}: GLU and SwiGLU operate on real-valued spatial tokens, while SAGU operates on complex-valued spectral coefficients $\hat{u} \in \mathbb{C}^L$; (2) \textbf{Gating mechanism}: GLU uses $x \odot \sigma(xW)$ and SwiGLU uses $x \odot \text{Swish}(xW)$ directly on inputs, whereas SAGU uses \textit{magnitude-based gating} $\sigma(|\hat{u}| W_2)$ to avoid phase discontinuities in the complex plane; (3) \textbf{Adaptivity}: The term "Adaptive" in SAGU emphasizes input-dependent spectral modulation—gates are computed from frequency magnitudes, enabling content-aware adjustment of spectral components. This design maintains the benefits of GLU-style architectures (improved gradient flow, enhanced expressiveness) while addressing the unique challenges of complex-valued frequency-domain processing.

\end{document}